\documentclass{article} %
\usepackage{conference,times}

\usepackage{amsmath,amsfonts,bm}

\def\eqref#1{equation~\ref{#1}}

\def\1{\bm{1}}

\DeclareMathAlphabet{\mathsfit}{\encodingdefault}{\sfdefault}{m}{sl}
\SetMathAlphabet{\mathsfit}{bold}{\encodingdefault}{\sfdefault}{bx}{n}

\usepackage{hyperref}
\usepackage{url}

\usepackage{verbatim}

\usepackage{comment}

\usepackage[utf8]{inputenc} %
\usepackage[T1]{fontenc}    %

\usepackage{url}            %
\usepackage{booktabs}       %
\usepackage{amsfonts}       %
\usepackage{nicefrac}       %
\usepackage{microtype}      %
\usepackage{xcolor}         %

\usepackage{xcolor}

\usepackage{booktabs}

\usepackage{comment}
\usepackage{booktabs}
\usepackage{multirow}
\usepackage{adjustbox}
\usepackage[table]{xcolor}
\usepackage{colortbl}
\usepackage{url}
\usepackage{xspace}
\usepackage{xcolor}
\usepackage{colortbl}

\usepackage[ruled,vlined]{algorithm2e}

\usepackage{multirow}
\usepackage{balance}
\usepackage{subfigure}
\usepackage{graphicx}
\usepackage{mdframed}
\usepackage{mathtools}
\usepackage{enumitem}
\usepackage[noabbrev,capitalise]{cleveref}

\usepackage{booktabs}
\usepackage{caption}   %
\usepackage{float}    %

\crefname{appendix}{Appendix}{appendices}
\Crefname{appendix}{Appendix}{Appendices}
\crefname{equation}{Eq.}{Eqs.}

\Crefname{equation}{Eq.}{Eqs.}

\crefname{lemma}{Lemma}{Lemmas}
\Crefname{lemma}{Lemma}{Lemmas}
\crefname{proposition}{Proposition}{Propositions}
\Crefname{proposition}{Proposition}{Propositions}
\crefname{assumption}{Assumption}{Assumptions}
\Crefname{assumption}{Assumption}{Assumptions}

\usepackage{amsthm} 
\usepackage{amsthm}
\newtheorem{proposition}{Proposition}
\newtheorem{lemma}{Lemma}
\theoremstyle{definition}
\newtheorem{assumption}{Assumption}
\theoremstyle{plain}
\usepackage{amsmath}
\usepackage{amsthm}
\usepackage{bm}
\usepackage{ragged2e}
\usepackage{colortbl}
\usepackage{booktabs}
\usepackage{colortbl}
\usepackage{xcolor}
\usepackage{amsmath, amssymb}
\usepackage{tcolorbox}
\usepackage{pifont}  %

\usepackage{xcolor} 
\usepackage{booktabs}
\usepackage{xcolor, booktabs, tcolorbox, amsmath, enumitem}
\tcbuselibrary{skins, breakable}

\usepackage{xcolor, booktabs, tcolorbox, amsmath, enumitem}
\tcbuselibrary{skins, breakable}

\usepackage{xcolor}
\usepackage{setspace}
\usepackage{listings}

\newcommand{\ourmodel}{USA\xspace}

\title{USA: Update-aware SAM for Cross-Domain\\ On-Policy Distillation of Language Agents}

\author{%
\hspace*{-\tabcolsep}%
\begin{minipage}{\textwidth}
\centering
\textbf{Qiyong Zhong}$^{1,2}$\thanks{Equal contribution.} \quad
\textbf{Mao Zheng}$^{2}$\footnotemark[1] \quad
\textbf{Mingyang Song}$^{2}$\footnotemark[1] \quad
\textbf{Huwei Ji}$^{3}$ \quad
\textbf{Houcheng Jiang}$^{1}$ \quad \\[0.35em]
\textbf{Jiajie Su}$^{3}$ \quad
\textbf{Li Zhang}$^{3}$ \quad
\textbf{Gengsheng Li}$^{2}$ \quad
\textbf{Junfeng Fang}$^{4}$\thanks{Corresponding author.} \\
\vspace{0.8em}
{\normalfont
$^{1}$University of Science and Technology of China \quad
$^{2}$Foundation Model Department, Tencent \\
$^{3}$Zhejiang University \quad
$^{4}$National University of Singapore \\
}
\vspace{0.8em}
{\normalfont\small
\texttt{\{youngzhong365,zhanglizl80\}@gmail.com} ;
\texttt{\{jihuwei,sujiajie\}@zju.edu.cn} \\
\texttt{\{moonzheng,nickmysong\}@tencent.com} ;
\texttt{ligengsheng2024@ia.ac.cn} \\
\texttt{jianghc@mail.ustc.edu.cn} ;
\texttt{fangjf@nus.edu.sg}
}
\end{minipage}%
\hspace*{-\tabcolsep}%
}

\conffinalcopy

\begin{document}

\maketitle

\begin{abstract}
On-policy distillation instils multi-turn agentic reasoning through dense token-level supervision
on the student's own trajectories, but a single domain saturates early, so further supervision has
to be drawn from other domains. Multi-domain data mixing is the most direct way of incorporating
them, at the cost of conflicts between their data distributions and of retraining the entire model
whenever one domain is revised. Model merging avoids both by distilling every domain independently
and fusing the resulting task vectors afterwards. We find instead that the benefit polarizes across
domain pairs: on those exhibiting negative transfer, every merging operator we evaluate falls below
the single-domain reference. We attribute this to cross-domain update coupling, where a substantial
fraction of coordinates is updated comparably by both domains and a merge can therefore displace
them by as much as their own updates. To overcome this limitation, we propose
\textbf{\underline{U}pdate-aware \underline{SA}M} (\ourmodel), which converts per-parameter update
magnitudes measured during a brief warm-up into per-coordinate perturbation radii, reducing
curvature precisely on the coordinates that carry most of the merging displacement. Experiments
across mathematics, science and code at two student scales show \ourmodel strongest in all six
transfer directions, ahead of the single-domain reference by more than four points on average, and
reverse the negative transfer of the conflicting pairs.
\end{abstract}

\section{Introduction}\label{sec:introduction}

Agentic capabilities let large language models solve complex mathematical, scientific, and programming tasks through multi-turn interaction with external tools~\citep{react,schick2023toolformer,xi2023rise}, and instilling them efficiently via post-training has become a central question~\citep{feng2025retool,li2025torl,xue2025simpletir}. Compared with reinforcement learning approaches such as GRPO~\citep{grpo}, which give only sparse outcome-level rewards, On-Policy Distillation (OPD)~\citep{opd,gu2024minillm} has a teacher supply dense token-level supervision on the student's own trajectories, yielding more stable optimization and higher sample efficiency. It has recently been introduced into agent post-training and shown to be effective~\citep{sod}.

Single-domain OPD, however, exhibits a pronounced performance ceiling. As illustrated in~\Cref{fig:motivation}(a) for Math, training saturates early and thereafter stagnates within a narrow band of fluctuation, and continually scaling up the distillation corpus brings no further gain, while a substantial gap to the teacher remains. This suggests that data volume alone is not the limiting factor, and points to the limited diversity of the supervision signal: the reasoning patterns within a single domain are relatively limited, so additional samples largely trigger patterns already covered by existing data and can hardly keep providing new learning signal. Since the bottleneck stems from the single-domain signal itself, incorporating knowledge from other domains as a complement becomes a natural choice.

Jointly training on a mixture of multi-domain data (mix-data) is the most direct way to incorporate cross-domain knowledge, but it has two inherent drawbacks. First, the data distributions of different domains conflict intrinsically~\citep{dong2024abilities}, so their optimization signals interfere during joint training~\citep{yu2020gradient} and end up hurting performance~\citep{mopd,wang2026decomposing}. Second, adding or updating any single domain requires inflexibly retraining the entire model. Model merging~\citep{Survery_ModelMerging_2024}, in contrast, fully decouples per-domain training: each domain trains independently from a shared initialization, and multi-source knowledge is then fused purely by parameter-space arithmetic~\citep{taskarithmetic}. This sidesteps data-distribution conflicts by construction, while domains can be developed in parallel and added or removed on demand.

\begin{figure*}[t]
\centering
\includegraphics[width=1.0\linewidth]{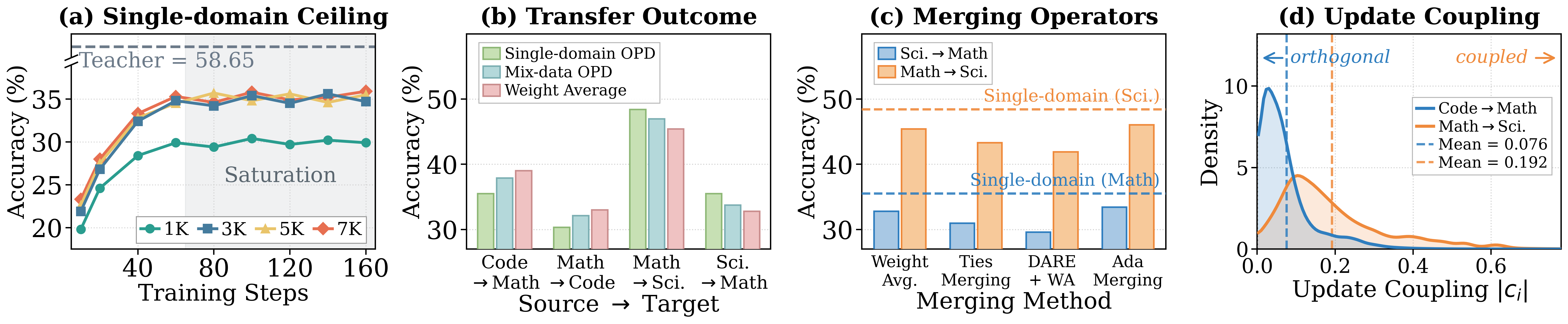}
\caption{\textbf{Motivation and diagnosis of cross-domain OPD via model merging.}
\textbf{(a)} OPD on Math saturates early, and scaling the corpus from 1K to 7K cannot break the ceiling.
\textbf{(b)} Gains polarize across domain pairs: Code$\leftrightarrow$Math improves in both directions, Math$\leftrightarrow$Science degrades in both.
\textbf{(c)} On the conflicting pair, all four merging operators fall below the single-domain baseline.
\textbf{(d)} Per-parameter update coupling $|c_i| = 2|\Delta_t(i)\Delta_s(i)| / (\Delta_t(i)^2 + \Delta_s(i)^2)$ with $\Delta_k = \theta_k - \theta_0$, being $0$ when only one domain updates it and $1$ when both do comparably. The facilitating pair concentrates near $0$, the conflicting pair shifts right. \Cref{app:coupling-magnitude} explains why it ignores the two signs.}
\label{fig:motivation}
\vspace{-1.4 em}
\end{figure*}

Yet however cross-domain knowledge is incorporated, the benefit polarizes across domain pairs. As shown in~\Cref{fig:motivation}(b), Code and Math form a \textbf{facilitating pair}: both mix-data and task-vector merging yield positive gains, surpassing the respective single-domain baselines in both directions. Math and Science, by contrast, form a \textbf{conflicting pair}, where both approaches suffer negative transfer in both directions and fall below the single-domain baselines. This degradation on the conflicting pair cannot be rescued by switching the merging operator: as shown in~\Cref{fig:motivation}(c), all four representative operators, weight averaging~\citep{wa}, Ties-Merging~\citep{ties}, DARE~\citep{dare} and AdaMerging~\citep{adamerging}, land below those baselines in both directions, indicating that the problem does not lie in any particular merging rule. We therefore examine the two kinds of domain pairs at the parameter level, where~\Cref{fig:motivation}(d) locates the boundary in the \textbf{per-parameter update coupling}: for the facilitating pair, the vast majority of parameters are substantially updated by only one domain, so the two sets of updates are nearly orthogonal and the other domain adds little to what merging already costs each of them; for the conflicting pair, a considerable fraction of parameters is updated by both domains with comparable magnitude and the coupling distribution shifts right as a whole, so on the coordinates through which the domains acquired their behaviour a linear combination is shaped as much by the other domain's update as by the domain's own, and the displacement it can impose there reaches the scale of the update itself. Negative transfer is therefore closely associated with \textbf{cross-domain update coupling}, which is already formed while the domains train independently and is bound to surface as mutual interference at merging time, before any operator has been chosen.

Mitigating this requires identifying which parameters carry the cross-domain conflict and making them resist merging interference already during independent per-domain training. To this end, we propose \textbf{\underline{U}pdate-aware \underline{SA}M} (\ourmodel), which builds on Sharpness-Aware Minimization (SAM)~\citep{foret2020sharpness} and adaptively allocates the flatness constraint by the update saliency of each parameter, so as to accommodate the parameter displacement that merging introduces. Prior work observes the effective updates of OPD to be highly sparse, with only a small subset of parameters carrying the acquired knowledge, and that subset to be locked in early in training~\citep{cai2026learning,shen2026geometry,yu2026dense}; we refer to them as \textbf{update-salient parameters}. Since merging linearly combines precisely the task vectors, these parameters contribute the bulk of the merging displacement and therefore dominate the magnitude of cross-domain interference. Accordingly, \ourmodel runs standard OPD for a few steps at the start of each domain's training, maps the measured per-parameter update magnitudes into scaling factors for the perturbation radius, and completes the remaining training within the ellipsoidal neighborhood they define. Compared with standard SAM, which treats all parameters alike, \ourmodel reallocates the flatness budget according to each parameter's contribution to interference: update-salient parameters must keep the loss low over a wider neighborhood, which reduces their local curvature and makes them far more tolerant to the displacement introduced by merging, while the remaining parameters fall back to the standard radius and do not sacrifice in-domain performance to needless regularization. Extensive experiments on agentic reasoning benchmarks across mathematics, science, and code show \ourmodel to consistently outperform all baselines and to reverse the negative transfer on conflicting pairs into positive gains.

\section{Preliminaries}\label{sec:preliminaries}

\subsection{On-Policy Distillation for Agent Post-Training}\label{subsec:opd}

We consider a general agent setting in which the model solves a task through
multi-turn interaction with an external environment. Given an input $x$, the policy
$\pi_\theta$ generates a segment $y_k$ at turn $k$, the environment returns an
observation $o_k$, and $o_k$ is appended to the context and conditions all
subsequent generation, until a final answer is produced. A complete trajectory is
therefore
\begin{equation}
\zeta = (x,\, y_1,\, o_1,\, \dots,\, y_K,\, o_K,\, y_{K+1}),
\label{eq:trajectory}
\end{equation}
where $K$ is the number of interaction turns. Since the observations are produced
by the environment rather than by the policy, they are excluded from the loss; we
write $\mathcal{T}$ for the set of token positions generated by the model, so that
supervision is applied on $\mathcal{T}$ only.

On-Policy Distillation (OPD) trains the student on trajectories sampled from the
student itself, while a teacher policy $\pi_{\text{teacher}}$ supplies dense
token-level supervision at every position of those trajectories. Concretely, it
minimizes a sampled estimate of the reverse KL divergence over the states the
student actually visits,
\begin{equation}
\mathcal{L}_{\text{OPD}}(\theta) = \mathbb{E}_{\zeta \sim \pi_\theta}
\left[\, \sum_{t \in \mathcal{T}}
\left( \log \pi_\theta(y_t \mid y_{<t}) - \log \pi_{\text{teacher}}(y_t \mid y_{<t}) \right)
\right].
\label{eq:opd}
\end{equation}
\Cref{eq:opd} is the training objective used for every domain throughout this
paper, and is also the objective executed during the warm-up stage of \ourmodel.

\subsection{Task Vectors and Model Merging}\label{subsec:merging}

We are given $K$ domains $\mathcal{D} = \{D_1, \dots, D_K\}$, each with its own teacher
and training data, and all starting from the same checkpoint $\theta_0$. Every domain $k$
is trained independently on its own data with the objective in \Cref{eq:opd}, updating all
parameters, which yields $\theta_k$. The knowledge acquired by domain $k$ is summarized by
its \emph{task vector}
\begin{equation}
\Delta_k = \theta_k - \theta_0,
\label{eq:task-vector}
\end{equation}
whose $i$-th component we denote by $\Delta_k(i)$.

Model merging then fuses the independently trained domains purely by arithmetic
operations in parameter space. Following weighted task arithmetic as the default
operator,
\begin{equation}
\theta_{\text{merged}} = \theta_0 + \sum_{k} \lambda_k \Delta_k,
\qquad \sum_{k} \lambda_k = 1, \quad \lambda_k \ge 0,
\label{eq:merging}
\end{equation}
so that every domain, including the one a given merge is evaluated on, contributes its task
vector through a coefficient rather than in full. Writing $D_t$ for the domain being evaluated,
the case of a single source domain $D_s$ reads
$\theta_{\text{merged}} = \theta_0 + (1 - \lambda)\Delta_t + \lambda \Delta_s$ with
$\lambda = \lambda_s$, and $\lambda = 1/2$ recovers plain weight averaging. An important
property of \Cref{eq:merging} is that merging acts \emph{coordinate-wise}: the final
value of the $i$-th parameter is determined solely by the linear superposition of
the task-vector components of the individual domains at that same position,
independently of all other coordinates.

\paragraph{Problem statement.}
Our goal is to make the merged model outperform on the target domain the model obtained by
training on target-domain data alone while keeping every domain trained independently, and
we report all results against that single-domain model as the reference point.

\section{Methodology}\label{sec:method}

\begin{figure*}[t]
\centering
\includegraphics[width=1.0\linewidth]{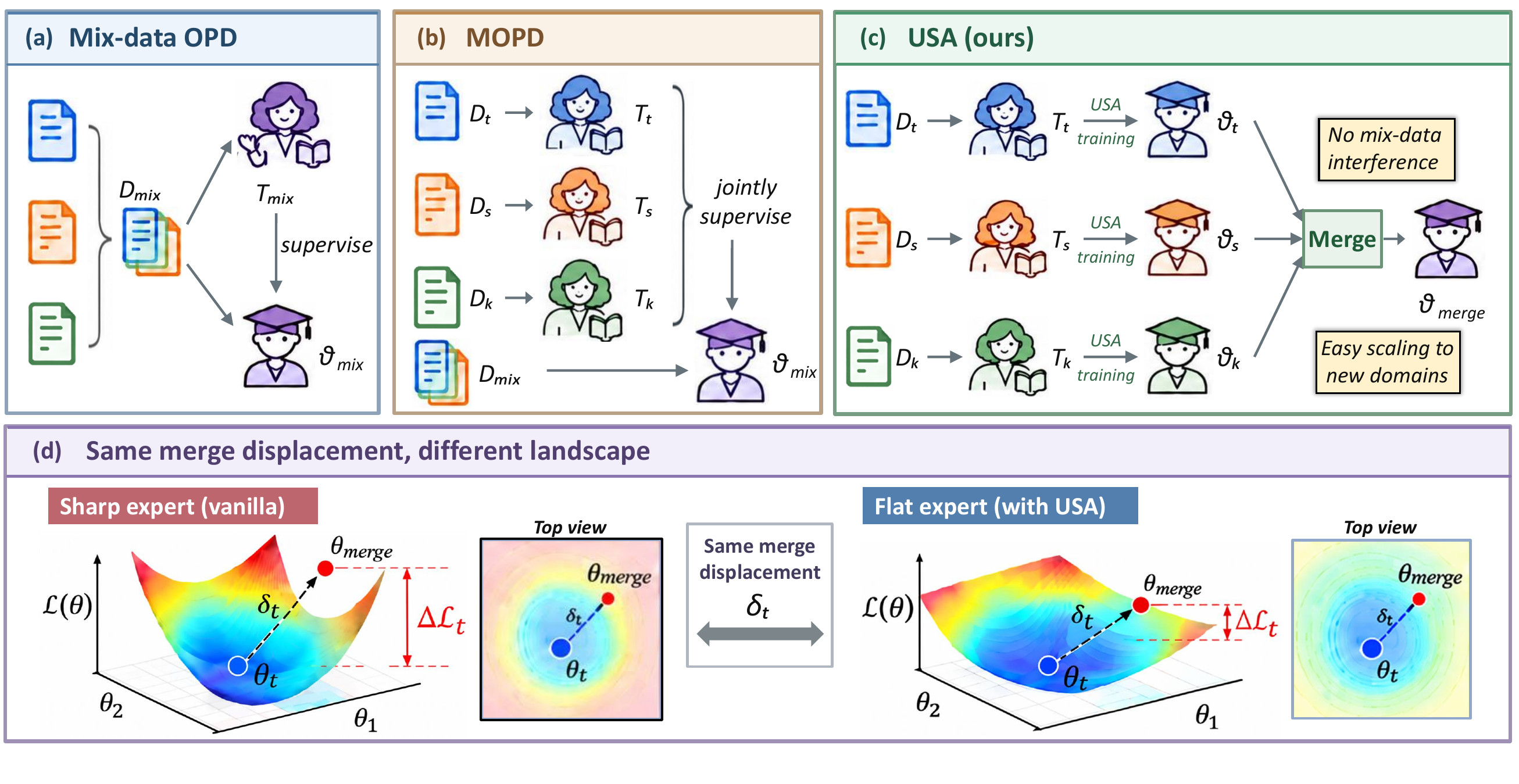}
\vspace{-2.5 em}
\caption{\textbf{Overview of \ourmodel and comparison with existing approaches.}}
\label{fig:framework}

\vspace{-0.5 em}
\end{figure*}

\subsection{What Merging Displaces}\label{subsec:locating}

\paragraph{Merging displaces each domain along the difference between the task vectors.}
As analyzed in~\Cref{sec:introduction}, negative transfer originates from the coupling of the
task vectors of different domains on shared parameters, a coupling already formed while the
domains are trained independently. \Cref{fig:framework} places the resulting scheme beside the
mixed-domain alternatives it is compared against. To act on it we ask what a domain is subjected to at merging
time. Fixing a target domain $D_t$ and subtracting its single-domain solution
$\theta_t = \theta_0 + \Delta_t$ from \Cref{eq:merging} leaves
\begin{equation}
\delta_t \;=\; \theta_{\text{merged}} - \theta_t \;=\; \sum_{j \neq t} \lambda_j \left( \Delta_j - \Delta_t \right) ,
\label{eq:displacement}
\end{equation}
the \emph{merging-induced displacement} of $D_t$, with one entry $\delta_t(i)$ per parameter,
which for a single source domain $D_s$ reduces to $\delta_t = \lambda(\Delta_s - \Delta_t)$.
Two of its properties are used below. The merged model evaluates $D_t$ at $\theta_t + \delta_t$
rather than at $\theta_t$, so any performance lost relative to the single-domain reference must be
attributed to $\delta_t$; the displacement itself decomposes coordinate-wise, although its effect
on the target loss does not, since that effect is governed by a quadratic form whose off-diagonal
terms couple the coordinates. Regrouping \Cref{eq:displacement} as
\begin{equation}
\delta_t \;=\; \underbrace{-\left(1 - \lambda_t\right) \Delta_t}_{\text{shrinkage of the target's own update}} \;+\; \underbrace{\sum_{j \neq t} \lambda_j \Delta_j}_{\text{injection from the other domains}}
\label{eq:decomposition}
\end{equation}
then separates a term the domain determines entirely by itself from one formed elsewhere that it
can neither evaluate nor influence while it trains. Both follow from \Cref{eq:merging} alone and
hold for any merging operator of this form regardless of how the domains were trained.

\paragraph{Which coordinates a domain has to protect.}
The injection term of \Cref{eq:decomposition} confines a domain to its own trajectory, since
intervening on the domains that form it would forfeit what merging gains over joint training: no
domain sees the data of another, and the domains train in parallel and can be added or removed one
at a time without retraining. The criterion must therefore come from the domain's own update,
which is also the one that matters: the coordinates it has moved far from $\theta_0$ are those
through which it acquired the behaviour the merged model is evaluated on, so a displacement of
given magnitude costs more there than on coordinates left untouched. The shrinkage term of
\Cref{eq:decomposition} bears on those coordinates by construction rather than by assumption,
because it scales $\Delta_t$ itself and is largest exactly where the domain has updated most, so
weight averaging pulls a domain back along the directions it relies on. What varies across domain
pairs is what the injection term does to this. \Cref{fig:motivation}(d) reports that on the pairs
motivating this work many parameters are updated by several domains with comparable magnitude, so
on the coordinates a domain relies on that term is comparable to the shrinkage term rather than
negligible beside it. What the criterion needs from it is its size and not its direction, because
the neighbourhood of \Cref{eq:objective} is symmetric about the solution and asks a coordinate to
tolerate a given magnitude of displacement whichever way it points, so a quantity that ignores
the two signs is the conservative input and $|c_i|$ is large exactly where the worst displacement
a merge can impose stands furthest above what the two updates alone would produce, as
\Cref{app:coupling-magnitude} sets out. Being a property of the domain
pairs rather than of the algebra above, this is stated as \Cref{asm:coupling}, under which the
stronger the coupling the larger the interference that can be removed.

\paragraph{Estimating the update magnitude with a short warm-up.}
Every domain is the target of the merges evaluated on it, so we drop $t$ and write $k$ for the
domain currently being trained. Its final update $|\Delta_k|$ is unavailable while it trains, and
this is where the training paradigm enters: recent work reports that the effective updates of OPD
are highly sparse and that the subset carrying them settles early, so a short warm-up already
identifies the coordinates to protect. We therefore let domain $k$ optimize \Cref{eq:opd} from
$\theta_0$ for $N$ steps, reaching $\theta_N^{(k)}$, and record the per-parameter
update magnitude
\begin{equation}
m_k(i) \;=\; \left| \theta_N^{(k)}(i) - \theta_0(i) \right| ,
\label{eq:magnitude}
\end{equation}
whose highest-ranked coordinates constitute the \textbf{update-salient parameters} introduced
in~\Cref{sec:introduction}. The magnitudes are computed once when the warm-up ends and kept fixed
thereafter, and they enter \Cref{subsec:perturbation} only after being normalized against a
reference formed from the domain's own updates, so the warm-up has to recover the relative profile
of the updates rather than their eventual sizes.

\subsection{Update-aware Adaptive Perturbation}\label{subsec:perturbation}

\paragraph{From displacement to a perturbation budget.}
Being displaced is not in itself harmful. What a domain loses is the increase in its own loss
between $\theta_t$ and $\theta_t + \delta_t$, which depends on the shape of the loss surface as
much as on the displacement, since a sharp minimum turns a small displacement into a steep climb
whereas a flat region absorbs the same displacement with little damage. Flatness is moreover the
one such property a domain governs through its own training, so \ourmodel adopts a
sharpness-aware objective and applies it along the coordinates its own update singles out.

Concretely, we turn the update magnitudes of \Cref{eq:magnitude} into a per-coordinate scale. Let
$g_i$ denote the weight matrix containing coordinate $i$. Within each matrix $g$ we take the
$(1-p)$-quantile of the update magnitudes as the reference level
\begin{equation}
D_k(g) \;=\; Q_{1-p}\!\left(\left\{\, m_k(i) \;:\; i \in g \,\right\}\right) ,
\label{eq:reference}
\end{equation}
against which every coordinate of that matrix is measured,
\begin{equation}
s_k(i) \;=\; 1 + (\alpha - 1)\,\min\!\left( \frac{m_k(i)}{D_k(g_i)},\; 1 \right) ,
\label{eq:scale}
\end{equation}
where $\alpha \ge 1$ caps the amplification so that $s_k(i) \in [1, \alpha]$ and $p$ fixes the
fraction of coordinates in each matrix that reach it in full, the remainder interpolating
continuously between $1$ and $\alpha$.

Three properties of this construction matter. It is invariant to the overall size of the update,
since replacing $m_k$ by $c\,m_k$ leaves \Cref{eq:scale} unchanged, so the warm-up of
\Cref{subsec:locating} has to recover the profile of the update rather than its eventual magnitude.
Its reference is a fixed fraction of the coordinates of one matrix rather than a single extreme
entry of the whole model, which is what makes the amplification usable on updates as heavy-tailed
as those of OPD: a reference at the largest entry would leave almost every coordinate at
$s_k(i) \approx 1$, and one shared across the network would deny amplification to any matrix whose
updates are uniformly small. It is finally continuous rather than a hard mask, so no
threshold has to be chosen and coordinates of intermediate saliency are treated in proportion. We
fix $p$ to $1\%$ throughout, in keeping with the sparsity the construction rests on, leaving
$\alpha$ as the only free quantity of \Cref{eq:scale}.

\paragraph{The training objective.}
Collecting the scales into the diagonal matrix $S_k = \mathrm{diag}(s_k)$, domain $k$
solves
\begin{equation}
\min_{\theta} \;\; \max_{\left\|S_k^{-1}\epsilon\right\|_2 \le \rho} \;\; \mathcal{L}_k(\theta + \epsilon) ,
\label{eq:objective}
\end{equation}
where $\rho$ is the base radius and $\mathcal{L}_k$ is the OPD objective of \Cref{eq:opd}. The
constraint set is an axis-aligned ellipsoid whose semi-axis along coordinate $i$ is
$\rho\,s_k(i)$, so the update-salient parameters have to keep the loss low over a wider interval
and the landscape is flattened along those directions, while coordinates of lower saliency are
constrained over an interval that shrinks continuously towards $\rho$ and retain their expressive
freedom. Setting $\alpha = 1$ gives
$S_k = I$ and recovers the isotropic ball, of which \ourmodel is therefore a strict generalization.

Substituting $\epsilon = S_k v$ turns the inner maximization after linearization into
$\max_{\|v\|_2 \le \rho} v^\top S_k \nabla_\theta \mathcal{L}_k(\theta)$, whose maximizer
$v^\star = \rho\, S_k \nabla_\theta \mathcal{L}_k(\theta) / \|S_k \nabla_\theta \mathcal{L}_k(\theta)\|$
gives the closed-form perturbation (\Cref{app:proof-perturbation} gives the full derivation)
\begin{equation}
\hat{\epsilon}_k \;=\; S_k v^\star \;=\; \rho\,\frac{S_k^2\, \nabla_\theta \mathcal{L}_k(\theta)}{\left\| S_k \nabla_\theta \mathcal{L}_k(\theta) \right\|} .
\label{eq:perturbation}
\end{equation}
Each step after the warm-up therefore evaluates the gradient at $\theta$, forms $\hat{\epsilon}_k$,
and updates $\theta$ with the gradient taken at the perturbed point $\theta + \hat{\epsilon}_k$.

\paragraph{Theoretical analysis.}
Write $\Xi_t = \mathcal{L}_t(\theta_t + \delta_t) - \mathcal{L}_t(\theta_t)$ for the \emph{merging
interference}, the loss merging inflicts on $D_t$ in excess of the solution it had reached on its
own. We analyse $\Xi_t$ as a local surrogate for the accuracy degradation reported in
\Cref{fig:motivation}(b), a positive value placing the merged model higher on the loss surface of
its target domain.
Expanding $\mathcal{L}_t$ around $\theta_t$ and bounding the resulting quadratic form in a diagonal
metric gives the following.

\begin{proposition}[Interference bound in a diagonal metric]
\label{prop:bound}
Let $\mathcal{L}_t$ be twice differentiable with a locally Lipschitz Hessian, write
$H_t = \nabla_\theta^2 \mathcal{L}_t(\theta_t)$, and let $\theta_t$ be a local minimizer of
$\mathcal{L}_t$, so that $\nabla_\theta \mathcal{L}_t(\theta_t) = 0$ and $H_t \succeq 0$. Then for
every diagonal $S$ with strictly positive entries,
\begin{equation}
\left|\Xi_t\right| \;\le\; \tfrac{1}{2}\,
\underbrace{\lambda_{\max}\!\left(S H_t S\right)}_{\textup{curvature in the } S\textup{-metric}}
\cdot
\underbrace{\left\|S^{-1}\delta_t\right\|_2^2}_{\textup{reweighted displacement}}
\;+\; \mathcal{O}\!\left(\|\delta_t\|_2^3\right) .
\label{eq:bound}
\end{equation}
\end{proposition}

Of the two factors the displacement is fixed once the domains have been trained, whereas the
curvature is a property of the loss surface that the domain shapes while training, as
\Cref{fig:framework}(d) illustrates, and it is
precisely what \Cref{eq:objective} suppresses: near a stationary local minimum its inner
maximization equals
$\tfrac{1}{2}\rho^2 \lambda_{\max}(S_k H_k S_k)$ to second order in $\rho$, the curvature factor of
the bound at $S = S_k$, so training and interference are expressed in one and the same metric
rather than related by analogy. \Cref{app:proof} proves \Cref{prop:bound}, establishes this
identity, and shows that the bound is reduced by placing the large entries of $S$ on the
coordinates that carry the displacement, which is what \Cref{eq:scale} does and
\Cref{asm:coupling} quantifies; it is also what limits how far the amplification may go, and is
why $\alpha$ caps it. \Cref{app:cost} reports the training cost and the interaction with the
surrounding merging pipeline.

\section{Experiment}\label{sec:experiment}

\subsection{Experimental Setup}\label{subsec:setup}

\paragraph{Datasets \& Benchmarks.}
We study three domains and train each on a $4$k dataset, drawn from DAPO-Math~\citep{yu2025dapo} for
mathematics, the code partition of Skywork-OR1~\citep{he2025skywork} for code, and
MegaScience~\citep{fan2025megascience} for science. Each is evaluated on two benchmarks: AIME 2025
and HMMT February 2026~\citep{dekoninck2026matharena}, LiveCodeBench-v6~\citep{jainlivecodebench}
and NaturalCodeBench~\citep{zhang-etal-2024-naturalcodebench}, and
GPQA-Diamond~\citep{rein2024gpqa} and SciBench-Atkins~\citep{wang2024scibench} respectively, with
the details in~\Cref{app:setup}.

\paragraph{Evaluation Setups.}
Each domain teacher is a Qwen3-14B model~\citep{qwen2025qwen3} further optimized with GRPO on the
corresponding dataset above, and the students are Qwen3-1.7B and Qwen3-4B. Unless stated otherwise,
every result we report is obtained on the Qwen3-1.7B student, so that the ablation
of~\Cref{subsec:ablation} and the analyses of~\Cref{subsec:gain,subsec:scalability} are read at one
scale. We decode with temperature $1.0$ and nucleus sampling at top\_p${=}0.6$, drawing $32$
samples per problem, and report average@$32$ as a percentage, with the agent placed in a code
interpreter throughout.

\paragraph{Baselines \& Implementation Details.}
We compare against (1) Vanilla, (2) Mix-data SFT, (3) Mix-data GRPO~\citep{grpo},
(4) Mix-data OPD, (5) Single-domain OPD and (6) MOPD~\citep{mopd}. We further compare against four
merging baselines, which start from the same OPD students that Single-domain OPD produces, each
trained on its own domain against its own teacher and without our objective, and combine their task
vectors by (7) Weight Average~\citep{wa}, (8) TIES-Merging~\citep{ties}, (9) DARE+WA~\citep{dare}
and (10) AdaMerging~\citep{adamerging}. These four therefore share identical experts and differ only
in the merging rule. \Cref{app:baselines} describes each baseline and \Cref{app:implementation}
reports the implementation details.

\begin{table*}[t]
    \centering
    \caption{
    \textbf{Cross-domain performance over six transfer directions at two student scales.}
    A column headed \emph{source} $\rightarrow$ \emph{target} evaluates the merged model on the
    target domain and averages its two benchmarks. The best results are in \textbf{bold}, and the
    second-best results are \underline{underlined}. We report average@32 over 5 runs, and a
    standard deviation is propagated from those of the two benchmarks, whose per-benchmark
    values are given in~\Cref{tab:detailed_1p7b,tab:detailed_4b}.
    }
    \vspace{0.5mm}
    \label{tab:main}

    \begingroup
    \setlength{\tabcolsep}{4.5pt}
    \renewcommand{\arraystretch}{0.95}
    \setlength{\aboverulesep}{1pt}
    \setlength{\belowrulesep}{1pt}
    \providecommand{\hdrstrut}{\rule[-3.9pt]{0pt}{13.0pt}}
    \normalsize

    \begin{adjustbox}{max width=0.98\textwidth}
    \begin{tabular}{lcccccc}
        \toprule

        {\footnotesize\textbf{Method}}\hdrstrut
        & {\footnotesize\textbf{Code $\rightarrow$ Math}}
        & {\footnotesize\textbf{Math $\rightarrow$ Code}}
        & {\footnotesize\textbf{Math $\rightarrow$ Sci.}}
        & {\footnotesize\textbf{Sci. $\rightarrow$ Math}}
        & {\footnotesize\textbf{Sci. $\rightarrow$ Code}}
        & {\footnotesize\textbf{Code $\rightarrow$ Sci.}} \\

        \midrule

        \rowcolor[HTML]{FEE090}
        \multicolumn{7}{c}{
            \hdrstrut\textbf{Teacher Model from the Qwen3-14B Series}
        } \\

        GRPO
        & 58.65\phantom{.}\scalebox{0.72}{$\pm$0.83}
        & 61.91\phantom{.}\scalebox{0.72}{$\pm$0.59}
        & 71.00\phantom{.}\scalebox{0.72}{$\pm$0.68}
        & 58.65\phantom{.}\scalebox{0.72}{$\pm$0.83}
        & 61.91\phantom{.}\scalebox{0.72}{$\pm$0.59}
        & 71.00\phantom{.}\scalebox{0.72}{$\pm$0.68} \\

        \midrule

        \rowcolor[HTML]{E0F3F8}
        \multicolumn{7}{c}{
            \hdrstrut\textbf{Student Models from the Qwen3-1.7B Series}
        } \\

        Vanilla
        & 15.09\phantom{.}\scalebox{0.72}{$\pm$0.71}
        & 19.03\phantom{.}\scalebox{0.72}{$\pm$0.46}
        & 36.41\phantom{.}\scalebox{0.72}{$\pm$0.56}
        & 15.09\phantom{.}\scalebox{0.72}{$\pm$0.71}
        & 19.03\phantom{.}\scalebox{0.72}{$\pm$0.46}
        & 36.41\phantom{.}\scalebox{0.72}{$\pm$0.56} \\

        Mix-data SFT
        & 26.00\phantom{.}\scalebox{0.72}{$\pm$0.95}
        & 22.32\phantom{.}\scalebox{0.72}{$\pm$0.46}
        & 41.40\phantom{.}\scalebox{0.72}{$\pm$0.65}
        & 25.07\phantom{.}\scalebox{0.72}{$\pm$0.83}
        & 22.20\phantom{.}\scalebox{0.72}{$\pm$0.41}
        & 42.49\phantom{.}\scalebox{0.72}{$\pm$0.64} \\

        Mix-data GRPO
        & 32.55\phantom{.}\scalebox{0.72}{$\pm$0.81}
        & 26.82\phantom{.}\scalebox{0.72}{$\pm$0.51}
        & 45.10\phantom{.}\scalebox{0.72}{$\pm$0.54}
        & 30.28\phantom{.}\scalebox{0.72}{$\pm$0.83}
        & 26.38\phantom{.}\scalebox{0.72}{$\pm$0.47}
        & 46.14\phantom{.}\scalebox{0.72}{$\pm$0.56} \\

        Mix-data OPD
        & 37.91\phantom{.}\scalebox{0.72}{$\pm$0.85}
        & 32.12\phantom{.}\scalebox{0.72}{$\pm$0.45}
        & 46.96\phantom{.}\scalebox{0.72}{$\pm$0.60}
        & 33.74\phantom{.}\scalebox{0.72}{$\pm$0.87}
        & 31.17\phantom{.}\scalebox{0.72}{$\pm$0.55}
        & 49.21\phantom{.}\scalebox{0.72}{$\pm$0.58} \\

        Single-domain OPD
        & 35.51\phantom{.}\scalebox{0.72}{$\pm$0.77}
        & 30.33\phantom{.}\scalebox{0.72}{$\pm$0.56}
        & 48.42\phantom{.}\scalebox{0.72}{$\pm$0.64}
        & 35.51\phantom{.}\scalebox{0.72}{$\pm$0.77}
        & 30.33\phantom{.}\scalebox{0.72}{$\pm$0.56}
        & 48.42\phantom{.}\scalebox{0.72}{$\pm$0.64} \\

        MOPD
        & 39.02\phantom{.}\scalebox{0.72}{$\pm$0.72}
        & \underline{33.12}\phantom{.}\scalebox{0.72}{$\pm$0.48}
        & \underline{48.85}\phantom{.}\scalebox{0.72}{$\pm$0.51}
        & \underline{35.64}\phantom{.}\scalebox{0.72}{$\pm$0.79}
        & \underline{32.07}\phantom{.}\scalebox{0.72}{$\pm$0.53}
        & \underline{50.21}\phantom{.}\scalebox{0.72}{$\pm$0.57} \\

        \cmidrule(lr){1-7}

        Weight Average
        & \underline{39.04}\phantom{.}\scalebox{0.72}{$\pm$0.58}
        & 33.01\phantom{.}\scalebox{0.72}{$\pm$0.42}
        & 45.44\phantom{.}\scalebox{0.72}{$\pm$0.72}
        & 32.80\phantom{.}\scalebox{0.72}{$\pm$0.86}
        & 31.54\phantom{.}\scalebox{0.72}{$\pm$0.37}
        & 50.05\phantom{.}\scalebox{0.72}{$\pm$0.68} \\

        Ties-Merging
        & 36.81\phantom{.}\scalebox{0.72}{$\pm$0.84}
        & 30.90\phantom{.}\scalebox{0.72}{$\pm$0.60}
        & 43.33\phantom{.}\scalebox{0.72}{$\pm$0.49}
        & 30.99\phantom{.}\scalebox{0.72}{$\pm$0.82}
        & 29.87\phantom{.}\scalebox{0.72}{$\pm$0.64}
        & 48.06\phantom{.}\scalebox{0.72}{$\pm$0.49} \\

        DARE+WA
        & 35.35\phantom{.}\scalebox{0.72}{$\pm$0.93}
        & 29.54\phantom{.}\scalebox{0.72}{$\pm$0.57}
        & 41.93\phantom{.}\scalebox{0.72}{$\pm$0.70}
        & 29.59\phantom{.}\scalebox{0.72}{$\pm$1.06}
        & 28.53\phantom{.}\scalebox{0.72}{$\pm$0.43}
        & 46.43\phantom{.}\scalebox{0.72}{$\pm$0.75} \\

        AdaMerging
        & 39.02\phantom{.}\scalebox{0.72}{$\pm$0.65}
        & 32.89\phantom{.}\scalebox{0.72}{$\pm$0.39}
        & 46.05\phantom{.}\scalebox{0.72}{$\pm$0.49}
        & 33.44\phantom{.}\scalebox{0.72}{$\pm$0.60}
        & 32.05\phantom{.}\scalebox{0.72}{$\pm$0.50}
        & 50.01\phantom{.}\scalebox{0.72}{$\pm$0.48} \\

        \cmidrule(lr){1-7}

        \textbf{\ourmodel}
        & \textbf{41.51}\phantom{.}\scalebox{0.72}{$\pm$0.83}
        & \textbf{35.64}\phantom{.}\scalebox{0.72}{$\pm$0.49}
        & \textbf{53.50}\phantom{.}\scalebox{0.72}{$\pm$0.53}
        & \textbf{38.76}\phantom{.}\scalebox{0.72}{$\pm$0.81}
        & \textbf{33.63}\phantom{.}\scalebox{0.72}{$\pm$0.44}
        & \textbf{52.80}\phantom{.}\scalebox{0.72}{$\pm$0.58} \\

        \midrule

        \rowcolor[HTML]{E0F3F8}
        \multicolumn{7}{c}{
            \hdrstrut\textbf{Student Models from the Qwen3-4B Series}
        } \\

        Vanilla
        & 36.79\phantom{.}\scalebox{0.72}{$\pm$0.81}
        & 39.54\phantom{.}\scalebox{0.72}{$\pm$0.50}
        & 51.50\phantom{.}\scalebox{0.72}{$\pm$0.60}
        & 36.79\phantom{.}\scalebox{0.72}{$\pm$0.81}
        & 39.54\phantom{.}\scalebox{0.72}{$\pm$0.50}
        & 51.50\phantom{.}\scalebox{0.72}{$\pm$0.60} \\

        Mix-data SFT
        & 41.71\phantom{.}\scalebox{0.72}{$\pm$0.87}
        & 45.15\phantom{.}\scalebox{0.72}{$\pm$0.49}
        & 55.07\phantom{.}\scalebox{0.72}{$\pm$0.66}
        & 40.64\phantom{.}\scalebox{0.72}{$\pm$0.87}
        & 44.83\phantom{.}\scalebox{0.72}{$\pm$0.49}
        & 55.65\phantom{.}\scalebox{0.72}{$\pm$0.57} \\

        Mix-data GRPO
        & 48.43\phantom{.}\scalebox{0.72}{$\pm$0.85}
        & 51.25\phantom{.}\scalebox{0.72}{$\pm$0.47}
        & 60.83\phantom{.}\scalebox{0.72}{$\pm$0.56}
        & 46.94\phantom{.}\scalebox{0.72}{$\pm$0.89}
        & 51.01\phantom{.}\scalebox{0.72}{$\pm$0.52}
        & 62.30\phantom{.}\scalebox{0.72}{$\pm$0.63} \\

        Mix-data OPD
        & 48.75\phantom{.}\scalebox{0.72}{$\pm$0.96}
        & 51.99\phantom{.}\scalebox{0.72}{$\pm$0.43}
        & 60.88\phantom{.}\scalebox{0.72}{$\pm$0.60}
        & 46.98\phantom{.}\scalebox{0.72}{$\pm$0.96}
        & 51.34\phantom{.}\scalebox{0.72}{$\pm$0.51}
        & 62.18\phantom{.}\scalebox{0.72}{$\pm$0.51} \\

        Single-domain OPD
        & 48.14\phantom{.}\scalebox{0.72}{$\pm$0.85}
        & 51.27\phantom{.}\scalebox{0.72}{$\pm$0.62}
        & \underline{61.71}\phantom{.}\scalebox{0.72}{$\pm$0.71}
        & 48.14\phantom{.}\scalebox{0.72}{$\pm$0.85}
        & 51.27\phantom{.}\scalebox{0.72}{$\pm$0.62}
        & 61.71\phantom{.}\scalebox{0.72}{$\pm$0.71} \\

        MOPD
        & \underline{49.58}\phantom{.}\scalebox{0.72}{$\pm$0.73}
        & \underline{52.68}\phantom{.}\scalebox{0.72}{$\pm$0.57}
        & 61.28\phantom{.}\scalebox{0.72}{$\pm$0.56}
        & \underline{48.63}\phantom{.}\scalebox{0.72}{$\pm$0.69}
        & \underline{52.22}\phantom{.}\scalebox{0.72}{$\pm$0.47}
        & \underline{62.70}\phantom{.}\scalebox{0.72}{$\pm$0.71} \\

        \cmidrule(lr){1-7}

        Weight Average
        & 49.32\phantom{.}\scalebox{0.72}{$\pm$0.93}
        & 52.53\phantom{.}\scalebox{0.72}{$\pm$0.51}
        & 59.02\phantom{.}\scalebox{0.72}{$\pm$0.60}
        & 46.09\phantom{.}\scalebox{0.72}{$\pm$0.81}
        & 51.91\phantom{.}\scalebox{0.72}{$\pm$0.44}
        & \underline{62.70}\phantom{.}\scalebox{0.72}{$\pm$0.52} \\

        Ties-Merging
        & 47.76\phantom{.}\scalebox{0.72}{$\pm$0.93}
        & 51.16\phantom{.}\scalebox{0.72}{$\pm$0.52}
        & 57.08\phantom{.}\scalebox{0.72}{$\pm$0.60}
        & 44.35\phantom{.}\scalebox{0.72}{$\pm$0.84}
        & 50.25\phantom{.}\scalebox{0.72}{$\pm$0.67}
        & 60.68\phantom{.}\scalebox{0.72}{$\pm$0.65} \\

        DARE+WA
        & 46.50\phantom{.}\scalebox{0.72}{$\pm$0.93}
        & 49.43\phantom{.}\scalebox{0.72}{$\pm$0.62}
        & 56.53\phantom{.}\scalebox{0.72}{$\pm$0.69}
        & 43.09\phantom{.}\scalebox{0.72}{$\pm$0.87}
        & 48.91\phantom{.}\scalebox{0.72}{$\pm$0.47}
        & 59.10\phantom{.}\scalebox{0.72}{$\pm$0.54} \\

        AdaMerging
        & 49.31\phantom{.}\scalebox{0.72}{$\pm$0.68}
        & 52.55\phantom{.}\scalebox{0.72}{$\pm$0.50}
        & 59.48\phantom{.}\scalebox{0.72}{$\pm$0.52}
        & 46.53\phantom{.}\scalebox{0.72}{$\pm$0.77}
        & 52.06\phantom{.}\scalebox{0.72}{$\pm$0.48}
        & 62.59\phantom{.}\scalebox{0.72}{$\pm$0.57} \\

        \cmidrule(lr){1-7}

        \textbf{\ourmodel}
        & \textbf{52.16}\phantom{.}\scalebox{0.72}{$\pm$0.92}
        & \textbf{56.25}\phantom{.}\scalebox{0.72}{$\pm$0.46}
        & \textbf{67.04}\phantom{.}\scalebox{0.72}{$\pm$0.54}
        & \textbf{50.72}\phantom{.}\scalebox{0.72}{$\pm$0.71}
        & \textbf{54.46}\phantom{.}\scalebox{0.72}{$\pm$0.45}
        & \textbf{65.65}\phantom{.}\scalebox{0.72}{$\pm$0.58} \\

        \bottomrule
    \end{tabular}
    \end{adjustbox}

    \endgroup

    \vspace{-10pt}
\end{table*}

\subsection{Overall Performance}\label{subsec:overall}

\Cref{tab:main} reports the six transfer directions at two student scales, with the per-benchmark
scores behind every cell in~\Cref{app:detailed}. \ourmodel is merged with the same weighted task
arithmetic as the Weight Average row, so any difference between the two is due entirely to how the
experts were trained.

\begin{itemize}[leftmargin=*]

\item \textbf{Obs 1: \ourmodel attains the best score in every transfer direction at both student
scales.} It improves on the single-domain reference everywhere, by $4.55$ points on average at
$1.7$B and $4.01$ points at $4$B, so the benefit is not confined to the smaller model where
headroom is largest.

\item \textbf{Obs 2: the benefit of cross-domain knowledge polarizes across domain pairs, and
\ourmodel is the only method that achieves positive transfer on the unfavourable pair in both
directions and at both scales.} On the facilitating pair, Code and Math, most methods help:
Mix-data OPD, MOPD, weight averaging and AdaMerging all clear the reference in all four cells,
the exceptions being Mix-data SFT and Mix-data GRPO, which average below it even here and so
indict their algorithms rather than their data. On the conflicting pair, Math and Science, the
four operators all fall below the reference at either scale and no other baseline clears it by as
much as a point. Mix-data OPD, which never merges, splits the same way, so the conflict is not the
merging arithmetic. On Math $\rightarrow$ Science the best operator falls $2.37$ points below the
reference while \ourmodel rises $5.08$ points above it, so cross-domain knowledge is not merely
kept from harming the target but made to benefit it.

\item \textbf{Obs 3: the improvement originates in training rather than in the merging rule or in
access to additional data.} \ourmodel merges by the same weighted task arithmetic as Weight
Average, yet surpasses AdaMerging, which optimizes its coefficients layer by layer, by close to
$4$ points at either scale, and surpasses in every direction both Mix-data OPD, which faces no
merging step at all, and MOPD. Reserving the flatness budget for the update-salient parameters
during independent training is therefore more effective than either reconciling the experts
afterwards or mixing the data.

\end{itemize}

\begin{table*}[t]
    \centering
    \caption{
    \textbf{Ablation study of \ourmodel on the Qwen3-1.7B students.}
    The best results are in \textbf{bold}. We report average@32 over 5 runs.
    }
    \vspace{0.5mm}
    \label{tab:ablation}

    \begingroup
    \setlength{\tabcolsep}{4.5pt}
    \renewcommand{\arraystretch}{0.95}
    \setlength{\aboverulesep}{1pt}
    \setlength{\belowrulesep}{1pt}
    \providecommand{\hdrstrut}{\rule[-3.9pt]{0pt}{13.0pt}}
    \normalsize

    \begin{adjustbox}{max width=0.98\textwidth}
    \begin{tabular}{lcccccc}
        \toprule

        {\footnotesize\textbf{Variant}}\hdrstrut
        & {\footnotesize\textbf{Code $\rightarrow$ Math}}
        & {\footnotesize\textbf{Math $\rightarrow$ Code}}
        & {\footnotesize\textbf{Math $\rightarrow$ Sci.}}
        & {\footnotesize\textbf{Sci. $\rightarrow$ Math}}
        & {\footnotesize\textbf{Sci. $\rightarrow$ Code}}
        & {\footnotesize\textbf{Code $\rightarrow$ Sci.}} \\

        \midrule

        \rowcolor[HTML]{E0F3F8}
        \multicolumn{7}{c}{
            \hdrstrut\textit{\textbf{Flatness Constraint}}
        } \\

        \quad (1.1) w/o SAM
        & 39.04\phantom{.}\scalebox{0.72}{$\pm$0.58}
        & 33.01\phantom{.}\scalebox{0.72}{$\pm$0.42}
        & 45.44\phantom{.}\scalebox{0.72}{$\pm$0.72}
        & 32.80\phantom{.}\scalebox{0.72}{$\pm$0.86}
        & 31.54\phantom{.}\scalebox{0.72}{$\pm$0.37}
        & 50.05\phantom{.}\scalebox{0.72}{$\pm$0.68} \\

        \quad (1.2) w/ Vanilla SAM
        & 40.28\phantom{.}\scalebox{0.72}{$\pm$0.77}
        & 34.66\phantom{.}\scalebox{0.72}{$\pm$0.56}
        & 50.21\phantom{.}\scalebox{0.72}{$\pm$0.63}
        & 35.94\phantom{.}\scalebox{0.72}{$\pm$0.92}
        & 32.70\phantom{.}\scalebox{0.72}{$\pm$0.48}
        & 51.55\phantom{.}\scalebox{0.72}{$\pm$0.52} \\

        \midrule

        \rowcolor[HTML]{FEE090}
        \multicolumn{7}{c}{
            \hdrstrut\textit{\textbf{Scale Assignment}}
        } \\

        \quad (2.1) w/ Random Scale
        & 39.68\phantom{.}\scalebox{0.72}{$\pm$0.90}
        & 33.74\phantom{.}\scalebox{0.72}{$\pm$0.62}
        & 47.86\phantom{.}\scalebox{0.72}{$\pm$0.79}
        & 34.25\phantom{.}\scalebox{0.72}{$\pm$0.63}
        & 32.05\phantom{.}\scalebox{0.72}{$\pm$0.55}
        & 50.72\phantom{.}\scalebox{0.72}{$\pm$0.43} \\

        \quad (2.2) w/ Inverted Scale
        & 38.71\phantom{.}\scalebox{0.72}{$\pm$1.08}
        & 32.82\phantom{.}\scalebox{0.72}{$\pm$0.48}
        & 44.62\phantom{.}\scalebox{0.72}{$\pm$0.87}
        & 31.95\phantom{.}\scalebox{0.72}{$\pm$0.99}
        & 31.18\phantom{.}\scalebox{0.72}{$\pm$0.68}
        & 49.67\phantom{.}\scalebox{0.72}{$\pm$0.72} \\

        \quad (2.3) w/ Hard Top-$k$ Mask
        & 40.92\phantom{.}\scalebox{0.72}{$\pm$0.65}
        & 35.11\phantom{.}\scalebox{0.72}{$\pm$0.39}
        & 52.34\phantom{.}\scalebox{0.72}{$\pm$0.49}
        & 37.91\phantom{.}\scalebox{0.72}{$\pm$0.75}
        & 33.24\phantom{.}\scalebox{0.72}{$\pm$0.42}
        & 52.19\phantom{.}\scalebox{0.72}{$\pm$0.61} \\

        \midrule

        \rowcolor[HTML]{E0F3F8}
        \multicolumn{7}{c}{
            \hdrstrut\textit{\textbf{Update Metric}}
        } \\

        \quad (3.1) w/ ASAM Scale
        & 40.17\phantom{.}\scalebox{0.72}{$\pm$0.84}
        & 34.78\phantom{.}\scalebox{0.72}{$\pm$0.44}
        & 50.35\phantom{.}\scalebox{0.72}{$\pm$0.69}
        & 35.82\phantom{.}\scalebox{0.72}{$\pm$0.59}
        & 32.81\phantom{.}\scalebox{0.72}{$\pm$0.61}
        & 51.41\phantom{.}\scalebox{0.72}{$\pm$0.47} \\

        \quad (3.2) w/ Relative-Change
        & 39.91\phantom{.}\scalebox{0.72}{$\pm$0.97}
        & 34.22\phantom{.}\scalebox{0.72}{$\pm$0.50}
        & 48.34\phantom{.}\scalebox{0.72}{$\pm$0.57}
        & 34.51\phantom{.}\scalebox{0.72}{$\pm$0.88}
        & 32.28\phantom{.}\scalebox{0.72}{$\pm$0.46}
        & 50.98\phantom{.}\scalebox{0.72}{$\pm$0.76} \\

        \midrule

        \textbf{\ourmodel}
        & \textbf{41.51}\phantom{.}\scalebox{0.72}{$\pm$0.83}
        & \textbf{35.64}\phantom{.}\scalebox{0.72}{$\pm$0.49}
        & \textbf{53.50}\phantom{.}\scalebox{0.72}{$\pm$0.53}
        & \textbf{38.76}\phantom{.}\scalebox{0.72}{$\pm$0.81}
        & \textbf{33.63}\phantom{.}\scalebox{0.72}{$\pm$0.44}
        & \textbf{52.80}\phantom{.}\scalebox{0.72}{$\pm$0.58} \\

        \bottomrule
    \end{tabular}
    \end{adjustbox}

    \endgroup

    \vspace{-10pt}
\end{table*}

\subsection{Ablation Study}\label{subsec:ablation}

\Cref{tab:ablation} alters one component of \ourmodel at a time, keeping the training pipeline and
the merging operator fixed. The spread between variants exceeds $8$ points on Math $\rightarrow$
Science against under $3$ on Science $\rightarrow$ Code, so it is on the conflicting pair that the
perturbation geometry decides the outcome.

\begin{itemize}[leftmargin=*]

\item \textbf{Obs 4: adaptive flatness is necessary, and uniform sharpness regularization is
insufficient.} An isotropic constraint lifts the average over the six directions by $2.24$ points
above the unconstrained merge, confirming that flatness is the right instrument, yet it remains
$1.75$ points below \ourmodel and barely reaches the single-domain reference on the conflicting
pair. Spending the budget uniformly leaves the coordinates that carry the interference
under-protected.

\item \textbf{Obs 5: the gain comes from assigning large scales to high-update coordinates, not
from anisotropy per se.} Permuting the scale vector keeps its marginal distribution but destroys
its link to the update magnitudes, which alone places it below the isotropic constraint
everywhere, and reversing it falls below the unconstrained merge. A two-level mask trails
\ourmodel by $0.69$ points, so identifying those coordinates supplies most of the gain and the
continuous map the rest.

\item \textbf{Obs 6: what a parameter has moved, rather than how large it is, determines merging
interference.} Deriving the scales from the parameter magnitude reproduces the isotropic result to
within noise, so that quantity carries no information about what merging displaces, and
normalizing the displacement by it is worse than leaving the geometry isotropic. Both placements
fall short of \ourmodel, which uses the displacement alone, as task arithmetic predicts: a merged
coordinate is an absolute sum of task-vector components, unaffected by the scale of the parameter
receiving them.

\end{itemize}

\subsection{Where the Cross-Domain Gain Comes From}\label{subsec:gain}

\Cref{fig:analysis}(a) decomposes the improvement of \ourmodel over the target-only baseline into
the effect of update-aware training on the target alone (Single-domain USA) and the additional
benefit of cross-domain merging (Cross-domain USA). \Cref{fig:analysis}(b) turns to the quantity
that distinguished the two kinds of domain pair in~\Cref{sec:introduction}, now between experts
trained with \ourmodel rather than plain OPD.

\begin{figure*}[t]
\centering
\includegraphics[width=1.0\linewidth]{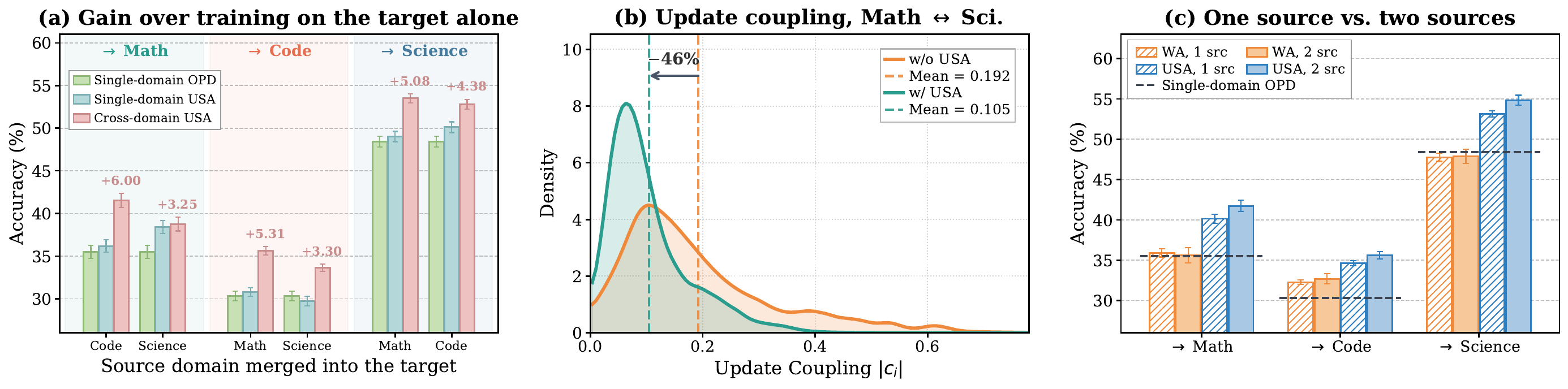}
\vspace{-1.9 em}
\caption{\textbf{Where the gain of \ourmodel comes from, and what happens when a second source
is added.}
(a) Each transfer direction under three training configurations, separating the gain of
update-aware training on the target from that of merging a source.
(b) Per-parameter update coupling $|c_i|$ of~\Cref{fig:motivation}(d) on
Math $\leftrightarrow$ Science, for experts trained by plain OPD and by \ourmodel.
(c) Both remaining domains merged into the target, against the single-source results
of~\Cref{tab:main}; \Cref{app:scaling} examines more. Error bars in (a) and (c) are standard
deviations over five runs.}
\label{fig:analysis}
\vspace{-0.5 em}
\end{figure*}

\begin{itemize}[leftmargin=*]

\item \textbf{Obs 7: the bulk of the gain comes from cross-domain knowledge, not from the training
mechanism itself.} Single-domain \ourmodel changes little and inconsistently, on two directions
falling slightly below the vanilla baseline. Once a source domain is merged in, Cross-domain
\ourmodel improves on the target-only reference in every direction, exceeding four points on four
of the six. The dominant factor is therefore the information merging brings in rather than a
stronger single-domain optimiser, and \Cref{eq:objective} pays off once several task vectors are
merged.

\item \textbf{Obs 8: \ourmodel leaves the two domains contending for markedly fewer parameters.}
Training the conflicting pair with \ourmodel reduces the mean update coupling by $46\%$. Both
curves are densities over the same support and enclose equal area, so the shift is not one being
flatter: mass moves from the coupled region towards zero, and the right tail formed by parameters
both domains move comparably is the part that thins, where the injection term
of~\Cref{eq:decomposition} contributes least.

\end{itemize}

\subsection{Scalability to More Source Domains}\label{subsec:scalability}

\Cref{fig:analysis}(c) merges both remaining domains into the target instead of one. Fixing the
target leaves exactly one choice of sources, so the three domains admit three such configurations,
each compared against the single-source result for the same target averaged over the two directions
of~\Cref{tab:main}. \ourmodel and Weight Average again differ only in how the experts were trained.

\begin{itemize}[leftmargin=*]

\item \textbf{Obs 9: a second source adds to the gain under \ourmodel and does not under plain
merging.} \ourmodel improves on its single-source result on all three targets and stays above the
target-only reference throughout, whereas Weight Average moves by less than the seed variation on
two targets and falls below its own single-source result on Math. Additional cross-domain knowledge
is thus available to be exploited rather than merely tolerated, once training has suppressed the
interference.

\end{itemize}

\paragraph{Further experiments.}
Four further studies are deferred to~\Cref{app:more-results} for space: merging a target with up
to five source domains~(\Cref{app:scaling}), sweeping the amplification cap $\alpha$
of~\Cref{eq:scale} and the base perturbation radius $\rho$ of~\Cref{eq:objective}
(\Cref{app:hparam}), replacing the merging operator with TIES, DARE and
AdaMerging~(\Cref{app:merger}), and repeating the comparison in a retrieval-based search-agent
environment~(\Cref{app:search}), in all of which \ourmodel stays strongest.

\section{Related Work}\label{sec:related}

\paragraph{On-policy Distillation.}
On-policy distillation~\citep{opdsurvey,gu2024minillm,opd} supervises the student token by token on
its own trajectories, removing the distribution mismatch of offline distillation, and has been
extended to self-supervised, sparse-reward, and multi-turn agentic
settings~\citep{opsd,srpo,sod}. Each run is nonetheless optimised alone, so nothing in its objective
anticipates coexistence with another domain's weights.

\paragraph{Model Merging.}
Merging~\citep{Survery_ModelMerging_2024} combines parameters instead of outputs, which works
because checkpoints from a common initialisation share a low-loss
region~\citep{frankle2020linear}, and task arithmetic~\citep{taskarithmetic} turns training
displacement into a vector that can be scaled and summed. Since these overlap, operators reduce
interference before summing~\citep{ties,dare,adamerging}. All act once the experts exist,
rearranging vectors formed without regard for the merge, whereas \ourmodel shapes them during
training; see~\Cref{app:related}.

\section{Conclusion}\label{sec:conclusion}
Merging independently distilled agentic experts helps on some domain pairs and harms others, and
the harm is decided by the parameters both domains update comparably, before any operator is
applied. \ourmodel therefore allocates the flatness budget of sharpness-aware training by update
magnitude, so the parameters carrying most of a merge's displacement are trained to tolerate it.
It attains the best result in every transfer direction we test and turns negative transfer into a
gain.

\bibliography{conference}
\bibliographystyle{conference}

\newpage
\appendix

\crefalias{section}{appendix}
\crefalias{subsection}{appendix}
\crefalias{subsubsection}{appendix}

\section*{Appendix}
\addcontentsline{toc}{section}{Appendix}

\section{Limitations}\label{app:limitations}

Two aspects of our experimental scope are worth stating. Every model belongs to the Qwen3 family,
which we chose because it is openly released, stable, published at many sizes and widely enough
used for our numbers to be placed against existing ones, so that the comparison across ten
baselines stays fair; a change of backbone is left to future work. The teachers of our main
experiments are moreover all $14$B models and we vary the student instead, since holding the
teacher scale fixed keeps each result differing from the others in the student alone, and how the
teacher scale interacts with the effect reported here is likewise left to future work.

\section{Related Work}\label{app:related}

\Cref{sec:related} introduces the two lines of work this paper builds on and states the limitation
of each one that motivates \ourmodel: on-policy distillation optimises every domain in isolation and
never anticipates that its weights will later be merged, while merging operators act only once the
experts already exist and therefore cannot change how their task vectors were formed. This appendix
gives the full account of both lines of work.

\subsection{On-policy Distillation.}\label{app:related-opd}

On-policy distillation~\citep{opdsurvey,opcd,gad,hybridOPD,scope,tip,chen2026soda} supervises the
student at the token level along trajectories the student itself produces, which removes the
distribution mismatch that makes offline distillation degrade once the student drifts away from the
teacher's data. \citet{gu2024minillm} casts the problem as reverse KL minimisation under the
student distribution, and \citet{opd} places on-policy and off-policy training in a single family
indexed by the divergence being minimised. Later analyses characterise what the objective actually
transfers: \citet{yang2026learning} reads it as KL-regularised reinforcement learning whose
per-token rewards are supplied implicitly by the teacher, while \citet{li2026rethinking} finds that
the gain comes from aligning the student on the states it visits, so that how much transfers depends
on how compatible the two reasoning styles are.

The same machinery has been turned inward, with the student supervising itself when no stronger
teacher is available~\citep{opsd,opsd2,sdpo,sdft,dqchen}, and outward, as a way of supplying the
dense signal that sparse-reward reinforcement learning
lacks~\citep{srpo,sdrlvr,bousselham2025vold,wang2026openclaw,skillsd}. The setting has also been
extended from single-turn responses to the multi-turn agents that this paper studies, where the
student interleaves reasoning with tool calls and is supervised step by step along trajectories
whose environment observations it did not itself produce~\citep{sod}. What this line of work has in
common is that each run is considered on its own: the objective is defined for a single teacher and
a single domain, and nothing in it anticipates that the resulting weights will later have to coexist
with those of another domain.

\subsection{Model Merging.}\label{app:related-merging}

Combining several specialised models by averaging their predictions~\citep{dong2020survey}
multiplies inference cost and does not carry over to language models, whose generated text admits no
meaningful average. Model merging~\citep{Survery_ModelMerging_2024} avoids both problems by
combining the parameters instead of the outputs, and it is workable because checkpoints fine-tuned
from a common initialisation tend to lie in a connected low-loss
region~\citep{1996Weight,10.5555/3327546.3327556,frankle2020linear}, so interpolating between them
need not destroy what either has learned. Averaging the corresponding weights is already effective
on its own~\citep{wa}, and task
arithmetic~\citep{taskarithmetic,task-arithmetic-lora} sharpens this into an algebra by treating the
displacement a model undergoes during fine-tuning as a vector that can be scaled, added, or
subtracted to compose and remove capabilities.

Because these vectors overlap, subsequent operators reduce the resulting interference before
summing~\citep{ji2026sharpness}, by trimming low-magnitude entries and reconciling conflicting
signs~\citep{ties}, by dropping and rescaling a large random fraction of the
entries~\citep{dare}, or by learning the coefficients on unlabelled test data~\citep{adamerging}.
More recent operators refine what a conflict is taken to be: \citet{gargiulo2025task} argues that
treating a task vector as a flat list of coordinates discards the structure of the network and
measures interference between layerwise singular directions instead, while \citet{sun2025cat}
removes the conflicting components of each vector before summing, projecting them out of linear
weights and masking them in normalisation parameters. The returns are nonetheless bounded:
\citet{wang2025model} reports that the benefit of adding experts decays as their number grows, so
most of what merging can deliver is already delivered by the first few. All of these operators,
however, are applied after the experts exist and can only rearrange task vectors that were produced
without any regard for the merge; \ourmodel intervenes earlier, shaping those vectors while each
expert is still being trained.

\section{Analysis of the Merging Interference}\label{app:proof}

This appendix supplies the material that \Cref{subsec:perturbation} defers. We derive the
interference that merging inflicts on a domain and prove the bound of \Cref{prop:bound}
(\Cref{app:proof-setup}), show that the training objective of \ourmodel controls exactly the
curvature factor of that bound (\Cref{app:proof-lemma}), and characterize when redistributing the
perturbation budget according to the update magnitudes tightens the estimate
(\Cref{app:proof-prop}). We are explicit throughout about which steps are exact and which are
assumptions.

\subsection{Setup and proof of the interference bound.}\label{app:proof-setup}

Let $D_t$ be the target domain. By \Cref{eq:displacement}, merging evaluates $D_t$ at
$\theta_t + \delta_t$ rather than at its own solution $\theta_t$, so the merging interference
of \Cref{subsec:perturbation} is the excess loss
\begin{equation}
\Xi_t \;=\; \mathcal{L}_t(\theta_t + \delta_t) - \mathcal{L}_t(\theta_t) ,
\label{eq:interference}
\end{equation}
and a positive $\Xi_t$ means the merged point sits higher on the loss surface of $D_t$ than the
model $D_t$ had trained on its own, which is the local surrogate through which we analyse the
accuracy degradation of \Cref{fig:motivation}(b).

Expanding $\mathcal{L}_t$ around $\theta_t$ to second order,
\begin{equation}
\Xi_t \;=\; \nabla_\theta \mathcal{L}_t(\theta_t)^\top \delta_t
\;+\; \tfrac{1}{2}\, \delta_t^\top H_t\, \delta_t
\;+\; \mathcal{O}\!\left(\|\delta_t\|_2^3\right) ,
\qquad H_t = \nabla_\theta^2 \mathcal{L}_t(\theta_t) .
\label{eq:taylor}
\end{equation}
Independent training drives $D_t$ to a local minimizer of its own objective, so
$\nabla_\theta \mathcal{L}_t(\theta_t) \approx 0$ and the linear term drops out, leaving
the curvature term as the leading contribution:
\begin{equation}
\Xi_t \;\approx\; \tfrac{1}{2}\, \delta_t^\top H_t\, \delta_t .
\label{eq:quadratic}
\end{equation}
Being a quadratic form, \Cref{eq:quadratic} grows both when the displacement is large and when
the loss surface is sharply curved along the direction in which the displacement points.

\begin{proof}[Proof of \Cref{prop:bound}]
Put $u = S^{-1}\delta_t$, so that $\delta_t = S u$ and, since $S$ is diagonal and hence
symmetric, $\delta_t^\top H_t \delta_t = u^\top (S H_t S) u$. Because $\theta_t$ is a local
minimizer we have $H_t \succeq 0$, hence $S H_t S \succeq 0$ and the quadratic form is
non-negative, so the Rayleigh quotient gives
$0 \le u^\top (S H_t S) u \le \lambda_{\max}(S H_t S)\|u\|_2^2$, where $\lambda_{\max}(\cdot)$ is
the largest eigenvalue. The local Lipschitz continuity of the Hessian makes the remainder in
\Cref{eq:taylor} cubic, and substituting $\|u\|_2 = \|S^{-1}\delta_t\|_2$ into
\Cref{eq:quadratic} and taking absolute values yields \Cref{eq:bound}.
\end{proof}

Two special cases are worth recording. Setting $S = I$ makes the change of variables the identity
and recovers
$|\Xi_t| \le \tfrac{1}{2}\lambda_{\max}(H_t)\|\delta_t\|_2^2 + \mathcal{O}(\|\delta_t\|_2^3)$,
the isotropic estimate that the standard sharpness analyses of model merging state. Replacing $S$
by $cS$ for any $c > 0$ multiplies $\lambda_{\max}(S H_t S)$ by $c^2$ and divides
$\|S^{-1}\delta_t\|_2^2$ by $c^2$, so the right-hand side of \Cref{eq:bound} is scale invariant
and only the relative profile of $S$ across coordinates can affect it.

\subsection{Derivation of the closed-form perturbation.}\label{app:proof-perturbation}

This subsection fills in the steps that the main text (from \Cref{eq:objective} to
\Cref{eq:perturbation}) compresses into one sentence. Everything below is a first-order argument
analogous to the one in the original SAM paper~\citep{foret2020sharpness}, the only difference
being that the isotropic ball $\|\epsilon\|_2 \le \rho$ is replaced by the ellipsoid
$\|S_k^{-1}\epsilon\|_2 \le \rho$.

\paragraph{Starting point.}
Each domain $k$ solves the min-max objective of \Cref{eq:objective}, reproduced here for
convenience:
\begin{equation}
\min_{\theta} \;\; \max_{\|S_k^{-1}\epsilon\|_2 \le \rho} \;\; \mathcal{L}_k(\theta + \epsilon) .
\tag{\ref{eq:objective}}
\end{equation}
The inner maximization asks for the worst-case perturbation within the ellipsoid, and the outer
minimization then drives the parameters towards a region where that worst case is mild.

\paragraph{Step 1: first-order approximation of the inner problem.}
Expand $\mathcal{L}_k(\theta + \epsilon)$ to first order about $\theta$:
\begin{equation}
\mathcal{L}_k(\theta + \epsilon)
\;\approx\; \mathcal{L}_k(\theta)
\;+\; \epsilon^\top \nabla_\theta \mathcal{L}_k(\theta) .
\label{eq:taylor-first}
\end{equation}
Because $\mathcal{L}_k(\theta)$ does not depend on $\epsilon$, the inner maximization reduces to
\begin{equation}
\hat{\epsilon}_k
\;=\; \arg\max_{\|S_k^{-1}\epsilon\|_2 \le \rho} \;\; \epsilon^\top \nabla_\theta \mathcal{L}_k(\theta) .
\label{eq:inner-linear}
\end{equation}

\paragraph{Step 2: change of variables.}
The diagonal matrix $S_k = \mathrm{diag}(s_k)$ has entries $s_k(i) \ge 1 > 0$, so it is
invertible. Setting $\epsilon = S_k v$, the constraint becomes
$\|S_k^{-1} S_k v\|_2 = \|v\|_2 \le \rho$, and the objective becomes
\begin{equation}
\max_{\|v\|_2 \le \rho} \; v^\top S_k \nabla_\theta \mathcal{L}_k(\theta) ,
\label{eq:inner-ball}
\end{equation}
which is the maximization of a linear function over an isotropic ball.

\paragraph{Step 3: closed-form solution via Cauchy--Schwarz.}
By the Cauchy--Schwarz inequality,
$v^\top (S_k \nabla_\theta \mathcal{L}_k(\theta)) \le \|v\|_2 \cdot \|S_k \nabla_\theta
\mathcal{L}_k(\theta)\|_2$, with equality when $v$ is parallel to
$S_k \nabla_\theta \mathcal{L}_k(\theta)$ and $\|v\|_2 = \rho$:
\begin{equation}
v^\star
\;=\; \rho\,\frac{S_k \nabla_\theta \mathcal{L}_k(\theta)}
{\left\| S_k \nabla_\theta \mathcal{L}_k(\theta) \right\|_2} .
\label{eq:vstar}
\end{equation}
Mapping back through $\hat{\epsilon}_k = S_k v^\star$:
\begin{equation}
\hat{\epsilon}_k
\;=\; \rho\,\frac{S_k^2\, \nabla_\theta \mathcal{L}_k(\theta)}
{\left\| S_k \nabla_\theta \mathcal{L}_k(\theta) \right\|} ,
\tag{\ref{eq:perturbation}}
\end{equation}
which is \Cref{eq:perturbation} of the main text.

\paragraph{Special case: isotropic SAM.}
When $\alpha = 1$, we have $S_k = I$ and \Cref{eq:perturbation} reduces to
$\hat{\epsilon} = \rho\, \nabla_\theta \mathcal{L}_k(\theta) /
\| \nabla_\theta \mathcal{L}_k(\theta) \|$,
which is the standard SAM perturbation~\citep{foret2020sharpness}. The derivation above
therefore specializes to the original one in the isotropic case, confirming that \ourmodel is a
strict generalization.

\paragraph{Update rule.}
Given $\hat{\epsilon}_k$, each step after the warm-up evaluates the gradient at the perturbed
point $\theta + \hat{\epsilon}_k$ and takes a descent step:
\begin{equation}
\theta \;\leftarrow\; \theta \;-\; \eta\, \nabla_\theta \mathcal{L}_k(\theta + \hat{\epsilon}_k) ,
\label{eq:update-rule}
\end{equation}
where $\eta$ is the learning rate. The perturbation $\hat{\epsilon}_k$ is never applied to the
parameters; its sole role is to move the evaluation point for the gradient so that the descent step
favours flat directions weighted by $S_k$.

\subsection{What the objective of \ourmodel controls.}\label{app:proof-lemma}

The next question is which quantity \Cref{eq:objective} actually minimizes. The answer holds to
second order in $\rho$ and requires no assumption beyond the expansion itself.

\begin{lemma}[Ellipsoidal sharpness]
\label{lem:sharpness}
Let $\mathcal{L}_k$ be twice differentiable at $\theta$ with a locally Lipschitz Hessian, let
$H_k = \nabla_\theta^2 \mathcal{L}_k(\theta)$, and suppose $\theta$ is a stationary local
minimizer of $\mathcal{L}_k$, so that $\nabla_\theta \mathcal{L}_k(\theta) = 0$ and
$H_k \succeq 0$. Then, to second order in $\rho$, the inner maximization of
\Cref{eq:objective} satisfies
\begin{equation}
\max_{\|S_k^{-1}\epsilon\|_2 \le \rho} \mathcal{L}_k(\theta + \epsilon) - \mathcal{L}_k(\theta)
\;=\; \tfrac{1}{2}\rho^2\, \lambda_{\max}\!\left(S_k H_k S_k\right)
\;+\; \mathcal{O}(\rho^3) .
\label{eq:lemma}
\end{equation}
\end{lemma}

\begin{proof}
Substitute $\epsilon = S_k v$. Since $S_k$ is diagonal with strictly positive entries it
is invertible, and the constraint $\|S_k^{-1}\epsilon\|_2 \le \rho$ becomes
$\|v\|_2 \le \rho$, so the change of variables is a bijection between the two feasible
sets. Expanding to second order and using stationarity,
\begin{equation}
\mathcal{L}_k(\theta + S_k v) - \mathcal{L}_k(\theta)
= \tfrac{1}{2}\, v^\top \!\left(S_k^\top H_k S_k\right) v + \mathcal{O}(\|v\|_2^3)
= \tfrac{1}{2}\, v^\top \!\left(S_k H_k S_k\right) v + \mathcal{O}(\|v\|_2^3) ,
\label{eq:lemma-proof}
\end{equation}
where $S_k^\top = S_k$ because $S_k$ is diagonal. Maximizing the quadratic form over the
ball $\|v\|_2 \le \rho$ is a Rayleigh-quotient problem whose optimum is attained at
$v = \rho u_{\max}$, with $u_{\max}$ a unit top eigenvector of $S_k H_k S_k$, giving
$\tfrac{1}{2}\rho^2 \lambda_{\max}(S_k H_k S_k)$, which is non-negative because
$H_k \succeq 0$ at a local minimizer and hence $S_k H_k S_k \succeq 0$. Collecting the remainder as
$\mathcal{O}(\rho^3)$ yields \Cref{eq:lemma}.
\end{proof}

Two consequences are worth stating plainly. First, at $\alpha = 1$ we have $S_k = I$ and
\Cref{eq:lemma} reduces to the familiar isotropic sharpness
$\tfrac{1}{2}\rho^2 \lambda_{\max}(H_k)$, so the objective of \ourmodel is a strict
generalization. Second, and this is the point \Cref{subsec:perturbation} appeals to, the
quantity being suppressed is not the curvature $H_k$ itself but the curvature \emph{as seen
through} $S_k$, which is the curvature factor of \Cref{eq:bound} evaluated at $S = S_k$.
Enlarging $s_k(i)$ inflates the $i$-th row and column of $S_k H_k S_k$, so coordinate $i$
contributes more to the eigenvalue being minimized and its curvature is driven down harder. The
perturbation geometry is therefore the mechanism by which the flatness budget is allocated across
coordinates.

\subsection{When the reweighting helps.}\label{app:proof-prop}

\Cref{eq:bound} is a family of \emph{upper} bounds on the same quantity $\Xi_t$, so a smaller
right-hand side does not by itself certify smaller interference and the choice of $S$ has to be
made deliberately. A two-dimensional example shows what is at stake. With $H_t = I$,
$\delta_t = (1, 0)^\top$ and $S = \mathrm{diag}(a, b)$, the right-hand side of \Cref{eq:bound}
equals $\max(a^2, b^2)/a^2$, which exceeds the isotropic value $1$ precisely when $b > a$, in
which case it is $b^2/a^2$. The reason is visible
in the two factors, since enlarging $S$ along a coordinate that is \emph{not} displaced buys
nothing in the displacement factor while still inflating the curvature factor. What matters is
therefore not that $S$ is non-uniform but that its large entries sit on the coordinates where
$|\delta_t(i)|$ is large, which is precisely how \Cref{eq:scale} is built. The following makes
this precise.

\begin{proposition}[Condition for a tighter estimate]
\label{prop:alignment}
Let $S = \mathrm{diag}(s)$ with $s(i) \ge 1$ for all $i$, write
$\hat{H} = S H_t S$, and suppose the curvature factor is unchanged by the reweighting,
$\lambda_{\max}(\hat{H}) = \lambda_{\max}(H_t)$. Then the estimate of \Cref{eq:bound} is
at least as tight as the isotropic one, with equality only if $s(i) = 1$ on every
coordinate where $\delta_t(i) \neq 0$. In general, writing $B_S$ for the quadratic leading
term of the right-hand side of \Cref{eq:bound}, that is for
$\tfrac{1}{2}\lambda_{\max}(S H_t S)\,\|S^{-1}\delta_t\|_2^2$ with the cubic remainder
excluded,
\begin{equation}
\frac{B_S}{B_I} \;=\; \frac{\lambda_{\max}(\hat{H})}{\lambda_{\max}(H_t)}
\cdot \frac{\sum_i \delta_t(i)^2 / s(i)^2}{\sum_i \delta_t(i)^2} ,
\label{eq:ratio}
\end{equation}
so the estimate improves precisely when the reduction of the second factor outweighs the
growth of the first.
\end{proposition}

\begin{proof}
The first claim is immediate from \Cref{eq:ratio}: if the leading factor equals one, the
second is a weighted average of $s(i)^{-2} \le 1$ with weights $\delta_t(i)^2$, hence at
most one, with equality only if $s(i) = 1$ whenever $\delta_t(i) \neq 0$.
\Cref{eq:ratio} itself follows by writing out both sides of \Cref{eq:bound} for $S$ and
for $I$ and dividing, using
$\|S^{-1}\delta_t\|_2^2 = \sum_i \delta_t(i)^2 / s(i)^2$.
\end{proof}

\Cref{eq:ratio} is the statement we rely on, and it says something narrower than the claim
that flatter is always better. The second factor is a $\delta_t^2$-weighted average of
$s(i)^{-2}$, so it is reduced by putting large $s(i)$ where $\delta_t(i)^2$ is large, whereas
raising $s(i)$ where $\delta_t(i)^2 \approx 0$ leaves that factor untouched and pays only the
cost in the first one. Since \Cref{eq:ratio} is a product, the second factor cannot be shrunk
indefinitely, because pushing $s(i)$ up eventually inflates $\lambda_{\max}(\hat{H})$, and the
budget can therefore only be \emph{redistributed} across coordinates rather than uniformly
reduced. This is the formal reason why the amplification in \Cref{eq:scale} is capped at
$\alpha$ rather than left unbounded. The scale invariance in \Cref{prop:bound} says the same
thing from another angle, because multiplying every $s(i)$ by a constant leaves \Cref{eq:bound}
untouched and only the relative profile of $s$ can matter, which is precisely the property
\Cref{subsec:locating} exploits when it argues that the warm-up has to recover the relative
profile of the update magnitudes rather than their eventual size.

\subsection{The coupling condition.}\label{app:proof-assumption}

\Cref{eq:ratio} asks for $s$ to be large where $|\delta_t|$ is large. Part of this is settled by
\Cref{eq:decomposition} without any further hypothesis, because the shrinkage term
$-(1 - \lambda_t)\Delta_t$ is proportional to the target's own task vector and is therefore largest
precisely on the coordinates that \Cref{eq:scale} amplifies, so the profile $s_k$ is aligned with at
least that part of the displacement by construction. The injection term
$\sum_{j \neq t} \lambda_j \Delta_j$ is formed by the other domains and admits no such guarantee,
and how much it contributes on those same coordinates is what distinguishes one domain pair from
another. \Cref{fig:motivation}(d) provides a proxy for this, and we state the condition explicitly.

\begin{assumption}[Cross-domain update coupling]
\label{asm:coupling}
On the domain pairs of interest, the two terms of \Cref{eq:decomposition} are of comparable
magnitude on the coordinates ranked highest by $m_t$: a non-negligible share of the mass of
$\big|\sum_{j \neq t} \lambda_j \Delta_j(i)\big|$ sits on those coordinates.
\end{assumption}

\Cref{fig:motivation}(d) is empirical evidence consistent with \Cref{asm:coupling}, since it
measures per parameter whether two domains update the same coordinates
with comparable magnitude and shows the mass shifted away from zero exactly on the pairs where
merging degrades performance. Two regimes follow. When the coupling is weak the injection term is
negligible on the update-salient coordinates and the displacement there is dominated by the
shrinkage term alone, which is the residual cost that weight averaging imposes on any domain.
When the coupling is strong the two terms are comparable in magnitude and the largest displacement
they can produce is correspondingly larger, so more interference is available to be removed and
\Cref{eq:ratio} admits a larger reduction. The assumption therefore does not decide whether
\Cref{eq:scale} points at the right coordinates, which \Cref{eq:decomposition} already settles, but
characterizes how much there is to be gained by protecting them.

\subsection{Why the coupling statistic does not separate the two signed cases.}\label{app:coupling-magnitude}
\Cref{fig:motivation}(d) reports $|c_i|$, which is unchanged if either update is replaced by its
negative, and this is a deliberate choice rather than an oversight. We first grant what the
statistic does not do. Writing the merge with equal coefficients, the displacement imposed on the
target at coordinate $i$ is $\delta_t(i) = \tfrac{1}{2}(\Delta_s(i) - \Delta_t(i))$ by
\Cref{eq:displacement}, so the two extreme configurations at $|c_i| = 1$ behave in opposite ways:
when $\Delta_s(i) = \Delta_t(i)$ the displacement vanishes, and when $\Delta_s(i) = -\Delta_t(i)$
it equals $|\Delta_t(i)|$, the full size of the target's own update. A large $|c_i|$ therefore
does not by itself imply a large realised displacement, and we do not claim that it does.

What $|c_i|$ does control is the worst case over the signs, which is the quantity the
construction actually needs. Holding the two updates' sizes fixed and maximizing over the unknown
signs gives $\max |\delta_t(i)| = \tfrac{1}{2}(|\Delta_t(i)| + |\Delta_s(i)|)$, and at fixed
$|\Delta_t(i)|$ this bound increases with $|c_i|$, since a larger coupling means the second term
is closer in size to the first. The neighbourhood of \Cref{eq:objective} is symmetric about the
solution and asks a coordinate to tolerate a displacement of a given magnitude whichever way it
points, so a bound that does not depend on the signs is precisely the right input to it, and a
signed statistic would not be usable in any case: the scales of \Cref{eq:scale} are fixed during
a domain's own warm-up, when the other domain has not finished training and its signs do not yet
exist.

There is also an empirical reason not to read the signs as the risk variable. On the pairs we
study, sign agreement is not what separates the good cases from the bad ones: both configurations
in which the two domains write to the same coordinate, whether their updates agree in sign or
oppose each other, transfer worse than the configuration in which the updates are nearly
orthogonal and each coordinate is claimed by one domain. What the merged model appears to be
sensitive to is that a coordinate is written twice, not the relative direction of the two writes,
so we treat the two signed cases as one phenomenon rather than distinguishing them. We report
this as an observation about these domain pairs rather than as a general law, and it is the
reason the diagnostic in \Cref{fig:motivation}(d) measures how far the two updates overlap in
size rather than how their directions relate.

Two limits of the statistic remain, and we state them rather than work around them.
\Cref{fig:motivation}(d) aggregates over all coordinates, whereas \Cref{asm:coupling} conditions
on those ranked highest by $m_t$, so the panel is evidence consistent with the assumption rather
than a measurement of it. And a displacement of given size costs a different amount of loss
depending on the local curvature, which is exactly why \Cref{prop:bound} keeps the curvature
factor and why \ourmodel acts on the geometry rather than on the displacement.

\paragraph{Scope of the claims.}
For clarity we separate what is proved from what is assumed. \Cref{lem:sharpness} and
\Cref{prop:bound,prop:alignment} are exact given (i) twice differentiability, (ii)
near-stationarity of $\theta_t$ after independent training, and (iii) truncation of the
expansions at second order, which is the same footing as the standard sharpness analyses of
model merging, of which \Cref{eq:bound} at $S = I$ is the special case. The alignment between
$s_k$ and the shrinkage term of \Cref{eq:decomposition} is algebraic and requires no assumption.
What we do
\emph{not} claim is that $s_k \propto m_k$ optimizes \Cref{eq:ratio}, since the optimum of a
product of a spectral quantity and a weighted average admits no clean closed form and the
argument above only establishes the direction in which $s_k$ should vary. \Cref{asm:coupling}
is an empirical property rather than a derived one, it bears on the size of the injection term
alone, and \Cref{fig:motivation}(d) is the evidence offered for it, with the limits of that
evidence set out in \Cref{app:coupling-magnitude}.

\section{Algorithm}\label{app:algorithm}

\Cref{alg:usa} summarizes the complete training and merging pipeline. Each domain is processed
independently; no cross-domain communication occurs until the task vectors are formed and merged
in the final step. The merging operator itself is plain weighted task arithmetic, inherited from
prior work and listed for completeness only, with the coefficients $\{\lambda_k\}$ fixed in
advance rather than learned.

\begin{algorithm}[h]
\caption{\ourmodel: Update-aware Sharpness-Aware Training for Cross-Domain Model Merging}
\label{alg:usa}
\KwIn{%
  Pre-trained model $\theta_0$;\\
  $K$ domain datasets $\{D_1, \ldots, D_K\}$;\\
  OPD objective $\mathcal{L}_k$ (\Cref{eq:opd});\\
  warm-up length $N$;\\
  reference fraction $p$;\\
  amplification cap $\alpha$;\\
  base perturbation radius $\rho$;\\
  merging coefficients $\{\lambda_k\}$.%
}
\KwOut{Merged model $\theta_{\mathrm{merged}}$.}

\vspace{0.5em}
\BlankLine
\textbf{Phase 1: Warm-up --- locate update-salient parameters (\Cref{subsec:locating})} \\
\For{each domain $k = 1, \ldots, K$ \textbf{in parallel}}{
    Initialize $\theta_k \leftarrow \theta_0$\;
    Train $\theta_k$ on $D_k$ for $N$ steps with the vanilla OPD objective $\mathcal{L}_k$ (\Cref{eq:opd})\;
    Compute update magnitudes: $m_k(i) = |\theta_k^{(N)}(i) - \theta_0(i)|$ for every parameter $i$\;
}

\vspace{0.5em}
\BlankLine
\textbf{Phase 2: Construct per-coordinate scales (\Cref{subsec:perturbation})} \\
\For{each domain $k = 1, \ldots, K$}{
    \For{each weight matrix $g$ in the model}{
        $D_k(g) \leftarrow Q_{1-p}\!\bigl(\{m_k(i) : i \in g\}\bigr)$\;
    }
    \For{each parameter $i$}{
        $s_k(i) \leftarrow 1 + (\alpha - 1)\,\min\!\bigl(m_k(i)\,/\,D_k(g_i),\; 1\bigr)$\;
    }
    $S_k \leftarrow \mathrm{diag}(s_k)$\;
}

\vspace{0.5em}
\BlankLine
\textbf{Phase 3: Update-aware SAM training (\Cref{subsec:perturbation})} \\
\For{each domain $k = 1, \ldots, K$ \textbf{in parallel}}{
    \For{each remaining training step}{
        Sample a batch of trajectories $\mathcal{B}_k \sim \pi_{\theta_k}$ and score it with the
        teacher, once\;
        Compute gradient $g_k \leftarrow \nabla_\theta \mathcal{L}_k(\theta_k; \mathcal{B}_k)$\;
        Form the worst-case perturbation:
        $\hat{\epsilon}_k \leftarrow \rho\, S_k^2\, g_k \,/\, \|S_k\, g_k\|$\;
        Update parameters: $\theta_k \leftarrow \theta_k - \eta\, \nabla_\theta \mathcal{L}_k(\theta_k + \hat{\epsilon}_k; \mathcal{B}_k)$\;
    }
    Form task vector: $\Delta_k \leftarrow \theta_k - \theta_0$\;
}

\vspace{0.5em}
\BlankLine
\textbf{Merging (\Cref{sec:preliminaries})} \\
$\theta_{\mathrm{merged}} \leftarrow \theta_0 + \sum_{k=1}^{K} \lambda_k\, \Delta_k$\;
\end{algorithm}

\section{Training Cost and Compatibility}\label{app:cost}

\subsection{Cost.}\label{app:cost-overhead}

Each optimization step of \ourmodel performs two forward-backward passes, one to form
$\hat{\epsilon}_k$ by \Cref{eq:perturbation} and one to take the gradient at
$\theta + \hat{\epsilon}_k$, which is the standard cost of a sharpness-aware update. The two passes
reuse a single batch of trajectories sampled once from the current policy, and the teacher scores
them once as well, so the update adds a second gradient computation without a second round of
generation. The
additional state is one copy of $\theta_0$ and one copy of $s_k$, both of which become read-only
once the warm-up ends, and no further state is introduced because \Cref{eq:magnitude} is evaluated
once and \Cref{eq:scale} is a fixed diagonal profile thereafter. Forming that profile costs one
quantile per weight matrix by \Cref{eq:reference} and is done a single time, which is negligible
against one training step. The warm-up itself optimizes
\Cref{eq:opd} unmodified, so it costs exactly the $N$ steps it consumes and nothing beyond them,
and $s_k$ is discarded when training ends since merging consumes task vectors alone.

The extra pass is nevertheless a small part of an OPD step, because such a step is dominated by
generation rather than by optimization. Before any gradient is taken the student has to roll out
long multi-turn trajectories in the agent environment and wait for a sandbox observation after every
tool call, and that generation phase consumes most of the wall-clock time of the step while the
forward-backward computation occupies a comparatively narrow slice of it. Doubling that slice
therefore moves the total by far less than it would in a conventional setting where the optimizer
dominates. Under the setting of \Cref{sec:experiment} the 1.7B student takes $44$ hours to train with
vanilla OPD and $50$ hours with \ourmodel, an increase of $13.6\%$, and the same comparison on the
4B student gives $11.8\%$. The relative overhead thus shrinks as the student grows, since a larger
student spends an even larger share of each step generating. Weighed against the gains reported
in \Cref{sec:experiment}, an overhead of roughly one part in eight is a cost we consider well spent.

\subsection{Compatibility with the merging pipeline.}\label{app:cost-scope}

\ourmodel leaves the surrounding pipeline intact. Domains are still trained in complete isolation
with no cross-domain communication, so they can be developed in parallel and added or removed one
at a time, and since \Cref{eq:objective} governs only how each domain is trained it remains
agnostic to the operator used later at merging time and to the direction in which the merge is
eventually formed. Each domain is therefore trained once and the resulting checkpoint serves every
merge it takes part in, whichever side of it the domain happens to be on.

This isolation is also what makes the overhead of the previous subsection worth paying, since the
alternative it replaces is more expensive in two respects. Training on a mixture of the domains
places their optimization signals in one run and lets them interfere, which is the effect
\Cref{sec:experiment} measures on the mix-data baselines, whereas \ourmodel never forms such a
mixture. A mixed run is moreover monolithic: a new domain changes the corpus and therefore requires
the whole model to be trained again, so the cost of accommodating $K$ domains is paid afresh every
time one of them is added or revised. Under \ourmodel a new domain costs one distillation run of
its own and leaves every existing student untouched, so the cost of an addition is independent of
how many domains are already in place. The merge that follows is a fixed number of arithmetic
operations over the task vectors, without gradients, rollouts or teacher queries, and takes minutes
against the days of \Cref{app:implementation}, so it does not enter the accounting in any practical
sense.

\section{Experimental Setup}\label{app:setup}

\subsection{Training Datasets.}\label{app:datasets}

\begin{table}[t]
    \centering
    \caption{
    \textbf{Summary of the training data of the main experiments.}
    Each domain is served by a single verifiable-reward dataset, used both as the RL corpus for
    that domain's teacher and as the prompt set for on-policy distillation, and the subsets are
    equally sized so that no domain is favoured by data volume when the task vectors are formed
    and merged. The three additional corpora that serve as extra sources in the scalability
    study are described in~\Cref{app:datasets} and are not listed here.
    }
    \vspace{1mm}
    \label{tab:dataset}

    \begingroup
    \setlength{\tabcolsep}{8pt}
    \renewcommand{\arraystretch}{1.08}
    \normalsize

    \begin{adjustbox}{max width=0.72\textwidth}
    \begin{tabular}{llc}
        \toprule

        {\footnotesize\textbf{Domain}}
        & {\footnotesize\textbf{Source}}
        & {\footnotesize\textbf{\#Samples}} \\

        \midrule

        Mathematical
        & DAPO-Math~\citep{yu2025dapo}
        & 4k \\

        Coding
        & Skywork-OR1 (Code)~\citep{he2025skywork}
        & 4k \\

        Scientific
        & MegaScience~\citep{fan2025megascience}
        & 4k \\

        \midrule

        \textbf{Total}
        &
        & \textbf{12k} \\

        \bottomrule
    \end{tabular}
    \end{adjustbox}

    \endgroup

    \vspace{-10pt}
\end{table}

Each domain is served by one dataset throughout, and \ourmodel uses no supervised fine-tuning data
at any point. The dataset first acts as the reinforcement learning corpus on which that domain's
Qwen3-14B teacher is optimized with GRPO, and it then supplies the prompts from which the student,
initialized at the released checkpoint $\theta_0$, samples its own trajectories during on-policy
distillation. Every source below therefore comes with programmatically verifiable answers, which is
what allows the same prompts to serve both stages: the reward for the teacher and the correctness
signal on the student's rollouts are computed by the same checker. We subsample $4$k prompts per
domain, summarized in~\Cref{tab:dataset}, and keep the three subsets equally sized so that no domain
enters the merge with a task vector inflated by data volume.

\par\smallskip
\noindent\textbf{Mathematical: DAPO-Math}~\citep{yu2025dapo}. DAPO-Math is the reinforcement
learning corpus released with the DAPO system and comprises roughly $17$k competition-style
mathematics problems collected from open sources. Every problem is paired with an integer final
answer, which makes the reward exactly verifiable by string match and avoids the noise that a
learned reward model or a judge would introduce. We draw our $4$k subset uniformly at random from
the deduplicated release.

\par\smallskip
\noindent\textbf{Coding: Skywork-OR1 (Code)}~\citep{he2025skywork}. Skywork-OR1 is a
reinforcement learning dataset spanning mathematics and code, of which we use the code partition
containing roughly $14$k programming problems drawn from competitive programming sources. Each
problem is accompanied by executable test cases, so correctness is determined by running the
generated program in a sandbox rather than by comparing text. We sample $4$k problems from this
partition and retain only those whose test harness executes successfully, resampling to keep the
subset size fixed.

\par\smallskip
\noindent\textbf{Scientific: MegaScience}~\citep{fan2025megascience}. MegaScience is a
large-scale post-training corpus for scientific reasoning, aggregating roughly $1.25$M instances
from open-source collections after decontamination and answer verification, and it covers physics,
chemistry and biology at university level. Its questions are open-ended and require applying
scientific principles and carrying out quantitative derivations rather than selecting among given
options. We draw $4$k instances uniformly at random, balanced across the three subject areas.

\par\smallskip
\noindent\textbf{Additional source domains.} The scalability study of \Cref{app:scaling}
enlarges the pool of source domains beyond the three above, and the three corpora described next
supply that extension. They are used as \emph{sources only}: no teacher is trained to be evaluated
on them, they appear in no transfer direction of \Cref{tab:main}, and no result in this paper is
reported on their benchmarks. They are therefore absent from \Cref{tab:dataset}, which summarizes
the data of the main experiments. Each is subsampled to $4$k instances, matching the three domains
above so that a task vector is not inflated by data volume, and each is filtered to the subset that
our Python interpreter can execute.

\par\smallskip
\noindent\textbf{Data Analysis: DataMind}~\citep{qiao2026datamind}.
DataMind is the training corpus released with the DataMind system and contains roughly 12k high-quality trajectories for generalist data-analytic agents, spanning diverse task categories and data formats such as CSV, XLSX, and structured databases.
Each trajectory records multi-turn reasoning interleaved with code execution and the resulting observations, making it naturally suited to tool-integrated post-training.
Since our environment exposes only a Python interpreter, we discard SQL-only trajectories and draw 4k instances uniformly at random from the remaining Python-executable subset, preserving the associated data files and execution traces.

\par\smallskip
\noindent\textbf{Financial: FinQA}~\citep{chen2021finqa}.
FinQA is a financial numerical reasoning corpus containing roughly 8k expert-annotated question--answer pairs grounded in approximately 2.8k financial reports, where each problem requires reasoning jointly over textual and tabular evidence.
Each instance is accompanied by supporting facts and a gold reasoning program whose execution produces the numerical answer, providing deterministic supervision without relying on a learned reward model or an LLM judge.
We use the official training split, remove duplicates and any instances overlapping with our evaluation sets, and uniformly sample 4k problems while retaining the original report context and executable reasoning annotations.

\par\smallskip
\noindent\textbf{Engineering: ERI}~\citep{naser2026eri}.
The Engineering Reasoning and Instruction (ERI) corpus is a large-scale instruction dataset for engineering reasoning whose public training split contains more than 45k instances across nine engineering disciplines and 55 subdomains, covering undergraduate, graduate, and professional difficulty levels.
Its tasks span several reasoning intents, including calculation and code-related problem solving, which naturally support numerical computation and equation solving with a Python interpreter.
To align the corpus with our tool-integrated setting, we retain only calculation and code-related instances and draw a 4k subset, approximately balanced across the nine engineering disciplines and difficulty levels.

\subsection{Benchmarks.}\label{app:benchmarks}

We provide detailed descriptions of the six benchmarks used throughout \Cref{sec:experiment}, two
per domain.

\subsubsection{Mathematics.}\label{app:benchmarks-math}

\noindent\textbf{AIME 2025.} The American Invitational Mathematics Examination (AIME) is a prestigious mathematics competition administered by the Mathematical Association of America (MAA). We use the 2025 edition, combining AIME I and AIME II, each containing 15 problems, for a total of 30 problems. All answers are integers in the range $[0, 999]$, which enables unambiguous automatic evaluation via exact match. The problems span advanced topics including algebra, number theory, combinatorics, and geometry, requiring multi-step reasoning chains and creative problem-solving strategies that go well beyond pattern matching. AIME I and AIME II were administered on February 6 and February 12, 2025, respectively. As a recent competition set, AIME 2025 can help reduce the risk of data contamination when evaluating models whose training data cutoff predates the release of these problems.

\par\smallskip
\noindent\textbf{HMMT Feb. 2026}~\citep{dekoninck2026matharena}. The Harvard--MIT Mathematics Tournament (HMMT) is a highly competitive high-school mathematics competition covering a broad range of advanced mathematical topics. We use the HMMT February 2026 benchmark released by MathArena, which is constructed from problems from the official February 2026 competition. MathArena extracted the problems from the official competition PDFs and manually verified and corrected the resulting transcriptions. Problems that were not suitable for final-answer evaluation, such as those requiring proofs or constructions, were excluded, resulting in a benchmark of 33 problems. The problems are categorized into algebra, number theory, combinatorics, and geometry, and each is associated with a gold final answer, enabling automatic evaluation through final-answer extraction. Because the benchmark is derived from a recent 2026 competition, it also provides a temporally fresh test set that can reduce contamination risk for models whose training data predates the competition. We evaluate on all 33 problems without any filtering. Since gold answers are frequently non-integer expressions such as fractions, radicals, and symbolic constants, a correct answer admits many syntactically different but mathematically identical forms; we therefore judge a prediction correct if it is mathematically equivalent to the gold answer rather than requiring a literal string match.

\subsubsection{Code.}\label{app:benchmarks-code}

\noindent\textbf{LiveCodeBench}~\citep{jainlivecodebench}. LiveCodeBench is a continuously updated benchmark for evaluating the coding capabilities of large language models. Unlike static benchmarks such as HumanEval, LiveCodeBench mitigates data contamination by continuously collecting new problems from competitive programming platforms (LeetCode, AtCoder, and Codeforces) and annotating each problem with its release date. This temporal annotation enables contamination-aware evaluation by restricting the test set to problems released after a model's training data cutoff. We use the v6 release, which contains 1,055 problems spanning May 2023 to April 2025. Following common practice, we evaluate on problems within a recent time window to reduce the risk of overlap with training data. The benchmark assesses multiple code-related capabilities including code generation, self-repair (debugging given execution feedback), code execution prediction, and test output prediction.

\par\smallskip
\noindent\textbf{NaturalCodeBench}~\citep{zhang-etal-2024-naturalcodebench}. NaturalCodeBench (NCB) is an application-driven code-generation benchmark designed to reflect the complexity and diversity of real-world programming requests. In contrast to benchmarks dominated by introductory algorithmic problems, NCB is constructed from natural user queries collected from online coding services. The benchmark comprises 402 high-quality problems in Python and Java across six domains: Software Engineering, Data Science, Algorithm and Data Structure, System Administration, Artificial Intelligence, and Front-End development. To support reliable execution-based evaluation for these realistic programming tasks, the authors introduce a semi-automated pipeline for constructing test cases, achieving more than a four-fold improvement in efficiency compared with fully manual construction. By evaluating generated programs against executable test cases, NaturalCodeBench provides a complementary assessment of practical code-synthesis capabilities beyond conventional competitive-programming or function-completion benchmarks. We evaluate on the Python \textit{dev} split, which covers task IDs 131--200 without gaps and therefore contains 70 problems, and score each generated program by executing it against the official test cases in an isolated sandbox. Before evaluation, we validate the split by running each problem's official reference solution against its own test cases, and we discard a problem if and only if the reference solution shipped with the benchmark fails them, which happens for the 12 problems with IDs 154, 160, 161, 165, 168, 170, 173, 175, 187, 188, 189 and 192. Every one of these failures is attributable to a defect in the released resources rather than to task difficulty, and they fall into two kinds. For nine of them the test cannot be executed at all, whatever program is supplied, because the data it reads was never released: eight (154, 161, 165, 168, 170, 173, 175, 189) depend on external input files absent from the distribution, and one (160) requires an external corpus that must be downloaded at runtime. For the remaining three the reference implementation is itself defective, two (187, 188) shipping solutions inconsistent with their own test cases and one (192) a syntactically invalid solution; here a failing reference shows only that the released solution does not pass, and we exclude these problems because they cannot be verified in a uniform environment, not because we claim no correct program could pass them. The remaining 58 problems form the evaluation set and are used without any further selection. The filter is mechanical, applied before any model was run and decided solely by the benchmark's own reference solutions, so it is independent of the outputs of the systems under study, and the same 58 problems are used for every method reported in this paper.

\subsubsection{Science.}\label{app:benchmarks-science}

\noindent\textbf{GPQA-Diamond}~\citep{rein2024gpqa}. GPQA (Graduate-Level Google-Proof Q\&A) is a challenging multiple-choice question-answering benchmark consisting of questions authored by domain experts holding PhDs in biology, physics, and chemistry. The questions are deliberately designed to be ``Google-proof''; they cannot be answered through simple web searches, requiring instead genuine domain expertise and multi-step scientific reasoning. The dataset comprises three nested subsets: GPQA Extended (546 questions), GPQA Main (448 questions), and GPQA Diamond (198 questions). We adopt the Diamond subset, which is the highest-quality partition: it contains only questions where both independent domain expert validators answered correctly, while skilled non-expert validators (holding PhDs in other scientific fields, with unrestricted internet access) failed. Human performance baselines on the Diamond subset are approximately 65\% for domain experts and 34\% for non-experts, compared to a random baseline of 25\%.

\par\smallskip
\noindent\textbf{SciBench-Atkins}~\citep{wang2024scibench}. SciBench is a benchmark designed to evaluate the ability of large language models to solve challenging college-level scientific problems requiring domain knowledge, multi-step reasoning, and advanced numerical computation. We use the \textit{Atkins} subset of SciBench, whose publicly released data files contain 105 problems drawn from \textit{Atkins' Physical Chemistry}. The subset focuses on physical chemistry and covers scientific concepts involving areas such as thermodynamics, equilibrium, molecular structure, and chemical reactions. Unlike multiple-choice science benchmarks, SciBench consists of open-ended, free-response problems that require models to retrieve and apply relevant scientific principles, select appropriate equations, and carry out potentially complex calculations to obtain the final answer. The dataset provides reference numerical answers and, where applicable, physical units, enabling systematic evaluation of models' quantitative scientific problem-solving capabilities. We evaluate on all 105 released problems without any filtering. Because each problem specifies its target unit separately from the problem statement, we state the required unit in the prompt so that the numerical answer is uniquely determined; without it a value that is physically correct but expressed on a different unit scale would be scored as an error.

\subsection{Baselines}
\label{app:baselines}

We compare \ourmodel against a diverse set of baselines covering direct multi-domain training, 
on-policy distillation, multi-teacher capability integration, and parameter-space model merging. 
Together, these baselines represent the main alternatives for incorporating knowledge from multiple 
domains: learning all domains jointly within a single model, integrating multiple teachers in policy 
space, or independently training domain experts and combining them afterward in parameter space. 
For the model-merging baselines, all domain experts are independently trained from the same 
initialization using single-domain OPD, and differ only in how their resulting task vectors are merged.

\begin{itemize}

\item \textbf{Vanilla.}
The base student model without any additional domain-specific post-training. 
This baseline measures the capabilities already present in the pretrained or instruction-tuned 
checkpoint and serves as the starting point for all subsequent student-side training.

\item \textbf{Mix-data SFT.}
A direct multi-domain supervised fine-tuning baseline in which the training examples from the 
source and target domains are pooled and used to fine-tune a single model. 
Unlike the independently trained experts used by model merging, the model is exposed to all 
participating domains simultaneously throughout training, allowing cross-domain knowledge to be 
integrated directly through shared parameter updates. 
This provides a standard offline-learning reference for evaluating whether simply increasing the 
diversity of supervised training data is sufficient to obtain positive cross-domain transfer.

\item \textbf{Mix-data GRPO}~\citep{grpo}.
A joint reinforcement-learning baseline in which a single policy is optimized on the mixture of 
source- and target-domain training problems using Group Relative Policy Optimization (GRPO). 
For each prompt, multiple responses are sampled from the current policy and assigned task-specific 
outcome rewards, from which GRPO constructs group-relative advantages without requiring a separate 
learned value function. 
The resulting policy-gradient objective therefore allows all participating domains to be optimized 
jointly under on-policy interaction, while supervision is provided primarily through sequence-level 
outcome signals rather than dense token-level guidance.

\item \textbf{Mix-data OPD}.
A joint on-policy distillation baseline that uses a single teacher and a single student across all 
participating domains. 
We first train a unified teacher with reinforcement learning on the mixture of source- and 
target-domain data. 
A single student is then trained on the same mixed-domain data using OPD: trajectories are sampled 
from the student policy, while the mixed-domain teacher supplies dense token-level supervision along 
the states visited by the student. 
This baseline isolates whether cross-domain integration can be achieved by exposing both the teacher 
and the student to the mixed data distribution, without introducing multiple domain-specific teachers 
or an explicit model-merging stage.

\item \textbf{Single-domain OPD}.
A target-only on-policy distillation baseline in which each domain is trained independently from the 
shared initialization using only its own training data and corresponding teacher. 
The student generates its own trajectories, and the domain teacher provides dense token-level 
supervision along those trajectories through the OPD objective. 
When evaluating a transfer direction, we use the model trained only on the target domain as the 
reference point. 
Accordingly, improvements over this baseline indicate positive cross-domain transfer, whereas lower 
performance indicates negative transfer.

\item \textbf{MOPD}~\citep{mopd}.
Multi-Teacher On-Policy Distillation integrates multiple domain-specialized teachers into a single 
student in policy space. 
The student is trained on multi-domain data and generates trajectories on-policy; each training 
example is routed to the teacher corresponding to its domain, which provides dense token-level 
supervision on the student's trajectory. 
In contrast to Mix-data OPD, which relies on one teacher jointly trained on the mixed-domain 
distribution, MOPD preserves separate domain-specialized teachers and combines their supervision 
during student training. 
It therefore serves as a strong distillation-based capability-integration baseline that does not rely 
on parameter-space model merging.

\item \textbf{Weight Average}~\citep{wa}.
A standard parameter-space merging baseline applied to independently distilled domain experts. 
Let $\theta_0$ denote the shared initialization and $\Delta_k = \theta_k - \theta_0$ the task vector 
of domain $k$. 
The merged model is obtained as $\theta_{\text{merged}} = \theta_0 + \sum_k \lambda_k \Delta_k$, 
where $\sum_k \lambda_k = 1$, so that every participating domain contributes its task vector through 
a merging coefficient. 
For two domains, equal weighting corresponds to 
$\theta_{\text{merged}} = \theta_0 + \frac{1}{2}\Delta_t + \frac{1}{2}\Delta_s$. 
This baseline represents the simplest form of model merging and provides the direct reference for 
measuring the effect of changing how the domain experts are trained while keeping the merging rule 
fixed.

\item \textbf{TIES-Merging}~\citep{ties}.
TIES-Merging is a task-vector merging method designed to reduce interference caused by redundant 
updates and sign conflicts across independently fine-tuned models. 
It consists of three steps: \emph{Trim}, which removes task-vector entries with small magnitudes; 
\emph{Elect Sign}, which determines a dominant update direction for each parameter coordinate 
across tasks; and \emph{Disjoint Merge}, which aggregates only the task-vector entries whose signs 
agree with the elected direction. 
By explicitly resolving conflicting parameter updates before aggregation, TIES provides a stronger 
alternative to direct weight averaging when independently trained domains disagree in parameter 
space.
Following the original work, we retain the top $20\%$ of entries by magnitude and use a merge
scaling of $\lambda = 1.0$.

\item \textbf{DARE+WA}~\citep{dare}.
DARE is a task-vector preprocessing technique that randomly drops a fraction of the parameter 
updates produced by fine-tuning and rescales the remaining updates to preserve their expected 
magnitude. 
Formally, DARE sparsifies each domain task vector before it is passed to the downstream merging 
operator. 
In our DARE+WA baseline, the processed task vectors are subsequently combined using the same 
weight-averaging rule described above. 
This baseline tests whether reducing redundant or unnecessary task-vector entries before merging is 
sufficient to mitigate cross-domain interference.
We sweep the drop rate over $p \in \{0.1, 0.2, \ldots, 0.9\}$ and report its best configuration.

\item \textbf{AdaMerging}~\citep{adamerging}.
AdaMerging replaces manually specified merging weights with coefficients learned from data. 
Starting from independently trained task vectors, it optimizes task-wise or layer-wise merging 
coefficients using an unsupervised objective, allowing different tasks and model components to 
contribute to the merged model with different strengths. 
Compared with fixed-weight averaging, AdaMerging therefore provides substantially more flexibility 
for reconciling heterogeneous domain experts in parameter space and serves as an adaptive 
model-merging baseline.
We use the layer-wise variant and initialize all learnable merging coefficients to $0.3$. Following
the original work, the coefficients are optimized via entropy minimization on unlabeled samples
using Adam with a learning rate of $1\times10^{-3}$, a batch size of $16$, and $500$ optimization
iterations.

\end{itemize}

\subsection{Implementation Details.}
\label{app:implementation}

All experiments are conducted using the VeRL framework~\citep{verl}. To ensure a fair comparison,
all methods share the same base training infrastructure, optimizer, and core hyperparameters unless
otherwise noted.

\begin{itemize}

    \item  \textbf{Agent Environment.}
Every method in this paper is trained and evaluated in the multi-turn agent setting
of~\Cref{subsec:opd}, in which the model interleaves reasoning with tool invocations and receives an
observation back before continuing. SandboxFusion serves as the code interpreter in every scenario,
so that the three domains differ in their data and their teacher but not in the environment they
interact with: a mathematical or scientific derivation is verified by executing code in the same
sandbox that runs a candidate program on its test cases. Sandbox observations are part of the
context but never of the loss, as \Cref{eq:opd} makes explicit, so a failed call influences training
only through the trajectory it leads the student into. The prompt templates that elicit this
behaviour are given in \Cref{app:prompt}. The one study that departs from this environment
is~\Cref{app:search}, which replaces the interpreter with a retrieval engine while keeping the
same treatment of observations.

    \item  \textbf{Compute Resources.}
All experiments are conducted on a single node with 8 NVIDIA H20 GPUs (96GB memory each). For RL and
distillation-based methods, training runs for 1 epoch and typically takes about 2 to 3 days for 1.7B
and 4B student models, while training with larger 14B teacher models requires approximately 5 days.
In contrast, supervised fine-tuning (SFT) is significantly more efficient, with 5 training epochs
typically completed within a few hours under the same hardware setup.
To ensure robustness and statistical reliability, all code-interpreter experiments are repeated
over 5 independent runs with different random seeds; the search-agent study of
\Cref{app:search} is a single-run generalization check and reports point estimates. These details
provide sufficient information about hardware configuration, execution time, and experimental
protocol for reproducibility.

    \item  \textbf{Baselines \& Teacher Models (RL \& Distillation).}
For all RL and distillation-based methods (Mix-data GRPO, Mix-data OPD, Single-domain OPD, MOPD)
including the teacher models, we adopt a unified training configuration to isolate the effect of each
algorithm. Specifically, we use the AdamW optimizer with a learning rate of 1e-6, a training batch
size of 64, and a mini-batch size of 16. The maximum prompt length is set to 2,560 tokens and the
maximum response length to 20,480 tokens. We sample 16 responses per prompt during training and 32
during validation. All methods are trained for at most 1 epoch (for the teacher models, we train for
at most 2 epochs). Rollout is performed asynchronously via vLLM with tensor parallelism of 4. The
teacher models of these experiments are Qwen3-14B further optimized with GRPO, and for
distillation-based methods the
teacher provides token-level supervision along the student's on-policy rollouts. Mix-data SFT is the
one baseline outside this configuration, as it involves no rollouts during training and therefore
requires reference responses, which none of our datasets provides in a form suitable for supervision.
We construct them from the same prompts used by every other method: each domain teacher samples
responses for the $4$k prompts of its own domain, only those accepted by that domain's verifier are
retained, and the resulting pairs are pooled across domains. Mix-data SFT then fine-tunes the base
model on this corpus for 5 epochs with a global batch size of 128, using AdamW with a learning
rate of 5e-5 and a maximum sequence length of 32,768 tokens with right truncation. It is also the
only method with an SFT stage, since every other method, \ourmodel included, starts from the released
base checkpoint $\theta_0$.

    \item  \textbf{Mixed-Domain Training and Supervision.}
The names of the baselines leave two questions open, namely which model is exposed to several domains
jointly and where its supervision comes from, and \Cref{tab:mixdata} answers both for every method.
Mix-data SFT and Mix-data GRPO query no teacher during optimization, training a single model
directly on the union of the participating domains' data; the offline targets Mix-data SFT is
fitted to are generated beforehand by the domain teachers, as described above. Mix-data OPD does use a teacher, and that teacher is itself the
Mix-data GRPO model, so the mixture is seen on both sides. MOPD mixes the student's data as well, but
keeps the single-domain teachers and routes each training example to the teacher of its own domain,
which is precisely what distinguishes it from Mix-data OPD. %
The four merging baselines involve no
mixed-domain training on either side: their task vectors are obtained exactly as \ourmodel obtains
its own, by running OPD separately per domain with that domain's teacher and data, and are then
combined by the respective merging operator. Weight Average uses the same weighted task arithmetic
as \ourmodel and therefore differs from it only in the objective under which each expert was
trained, whereas the remaining three additionally modify how the experts are combined. Our main
experiments instantiate this with one source
domain per target domain, which isolates the pairwise effect that \Cref{subsec:locating} analyzes,
while every method above, \ourmodel included, is defined for an arbitrary number of domains.

\begin{table}[t]
    \centering
    \caption{
    \textbf{Mixed-domain training and supervision across the compared methods.}
    A method is \emph{mix-data} on a given axis when the model in question is exposed to the
    participating domains jointly rather than to one domain at a time. \ourmodel and the four
    merging baselines share every column of this table, and \ourmodel and Weight Average also
    share the merging rule, so those two differ only in the objective under which each expert
    is trained.
    }
    \vspace{1mm}
    \label{tab:mixdata}

    \begingroup
    \setlength{\tabcolsep}{5pt}
    \renewcommand{\arraystretch}{1.12}
    \normalsize

    \begin{adjustbox}{max width=\textwidth}
    \begin{tabular}{lcccc}
        \toprule

        \multirow{2}{*}{{\footnotesize\textbf{Method}}}
        & \multicolumn{2}{c}{{\footnotesize\textbf{Mixed-domain data}}}
        & \multirow{2}{*}{{\footnotesize\textbf{Teacher}}}
        & \multirow{2}{*}{{\footnotesize\textbf{Merging}}} \\

        \cmidrule(lr){2-3}

        & {\footnotesize Student}
        & {\footnotesize Teacher}
        & & \\

        \midrule

        Mix-data SFT
        & \checkmark
        & --
        & none
        & -- \\

        Mix-data GRPO
        & \checkmark
        & --
        & none
        & -- \\

        Mix-data OPD
        & \checkmark
        & \checkmark
        & Mix-data GRPO
        & -- \\

        MOPD
        & \checkmark
        & \ding{55}
        & single-domain, routed per example
        & -- \\

        \midrule

        Single-domain OPD
        & \ding{55}
        & \ding{55}
        & single-domain
        & -- \\

        \midrule

        Weight Average
        & \ding{55}
        & \ding{55}
        & single-domain
        & \checkmark \\

        TIES-Merging
        & \ding{55}
        & \ding{55}
        & single-domain
        & \checkmark \\

        DARE+WA
        & \ding{55}
        & \ding{55}
        & single-domain
        & \checkmark \\

        AdaMerging
        & \ding{55}
        & \ding{55}
        & single-domain
        & \checkmark \\

        \midrule

        \textbf{\ourmodel}
        & \ding{55}
        & \ding{55}
        & single-domain
        & \checkmark \\

        \bottomrule
    \end{tabular}
    \end{adjustbox}

    \endgroup

    \vspace{-10pt}
\end{table}

    \item  \textbf{\ourmodel.}
\ourmodel adopts the same training configuration as the RL and distillation baselines described
above, is trained per domain with the same single-domain teacher and data, and is merged with the
same weighted task arithmetic as the Weight Average baseline, under equal coefficients that are
fixed rather than tuned or learned, so that two domains are combined as
$\theta_{\text{merged}} = \theta_0 + \frac{1}{2}\Delta_t + \frac{1}{2}\Delta_s$. The warm-up of \Cref{eq:magnitude}
runs for $N = 10$ steps, which is an absolute step count rather than a fraction of the training
budget because the subset of parameters carrying the update settles after a number of steps that
does not scale with how long a domain is trained. The reference level of \Cref{eq:reference} is
taken within each weight matrix at $p = 1\%$, the base perturbation radius is $\rho = 0.03$, and the
amplification cap is $\alpha = 5$, so that the per-coordinate scales of \Cref{eq:scale} satisfy
$s_k(i) \in [1, \alpha]$. This is a single global setting, used unchanged across all domains and
both model scales rather than tuned per domain, and \Cref{app:hparam} reports a sensitivity
analysis around it. One-dimensional parameters
such as the LayerNorm gains are left unscaled
at $s_k(i) = 1$, since they hold too few entries for a $1\%$ reference level to be meaningful. The
update magnitudes are computed once when the warm-up ends and kept fixed for the remainder of
training. Denominators are floored by a small positive constant throughout, both the reference
level of \Cref{eq:reference} and the norm of \Cref{eq:perturbation}, so that a matrix whose
reference level vanishes leaves its coordinates at $s_k(i) = 1$ rather than producing a division by
zero. Of the two, only $\alpha$ is specific to \ourmodel: $\rho$ is inherited from
sharpness-aware training and searched over the same grid for every sharpness-aware variant, while
$N$ and $p$ are held at the values stated here throughout.

\end{itemize}

\section{More Experimental Results}\label{app:more-results}

\subsection{Scaling the Number of Source Domains.}\label{app:scaling}

\begin{figure*}[h]
\centering
\includegraphics[width=1.0\linewidth]{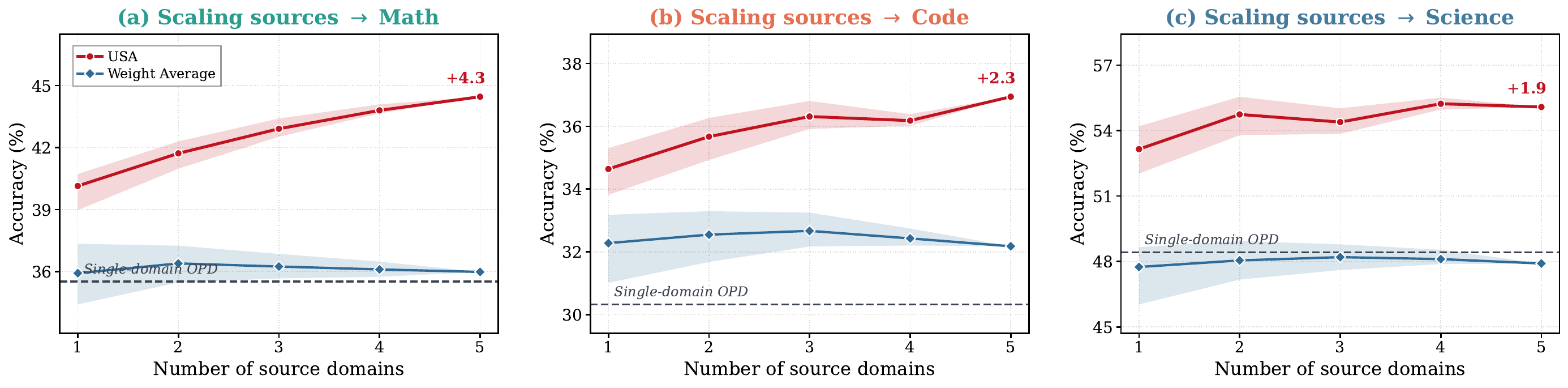}
\caption{\textbf{Accuracy against the number of source domains merged into a fixed target, on
the Qwen3-1.7B students.} One panel per target domain, with \ourmodel and Weight Average
merging the same experts and differing only in how those experts were trained. Each point
averages every combination of sources of that size and the band spans the weakest and the
strongest of them, so it is a range over the choice of sources rather than a seed interval.
The dashed line is the Single-domain OPD reference of~\Cref{tab:main} and the label gives the
change of \ourmodel from one source to five.}
\label{fig:scaling}
\vspace{-0.5 em}
\end{figure*}

\Cref{fig:analysis}(c) merges two sources into a target, which is as far as the three domains
of~\Cref{tab:main} allow. To go further we extend the pool with three additional domains, Data
Analysis, Finance and Engineering, drawn from the corpora described in~\Cref{app:datasets}. These
three serve only as sources and are never evaluated on, so a fixed target admits up to five of
them and the study covers three targets rather than six. Each point of~\Cref{fig:scaling} is the
mean over all combinations of sources of that size, which removes any dependence on a particular
choice and leaves a single combination once all five are merged. \ourmodel and Weight Average merge
by the same weighted task arithmetic here as everywhere else in the paper, so the distance between
the two curves is again attributable to how the experts were trained.

\begin{itemize}[leftmargin=*]

\item \textbf{Obs 10: the two methods differ in trend rather than only in level.} \ourmodel gains
over the whole range on every target, by $4.3$, $2.3$ and $1.9$ points from one source to five,
while Weight Average turns over after two sources on Math and after three on Code and Science and
then declines towards its starting point. Neither curve is monotone, since a point averages the
combinations available at that size and the pool is small, but the direction over the range is
unambiguous and holds on all three targets. What the scaling behaviour adds to~\Cref{fig:analysis}(c)
is therefore not a larger gain but a different shape: under plain merging the amount of
cross-domain knowledge a target can absorb is bounded and reached early, whereas suppressing the
interference during training leaves the merged model able to take in more of it than the pool
contains.

\end{itemize}

\subsection{Hyperparameter Sensitivity Analysis.}\label{app:hparam}

\begin{figure*}[h]
\centering
\includegraphics[width=1.0\linewidth]{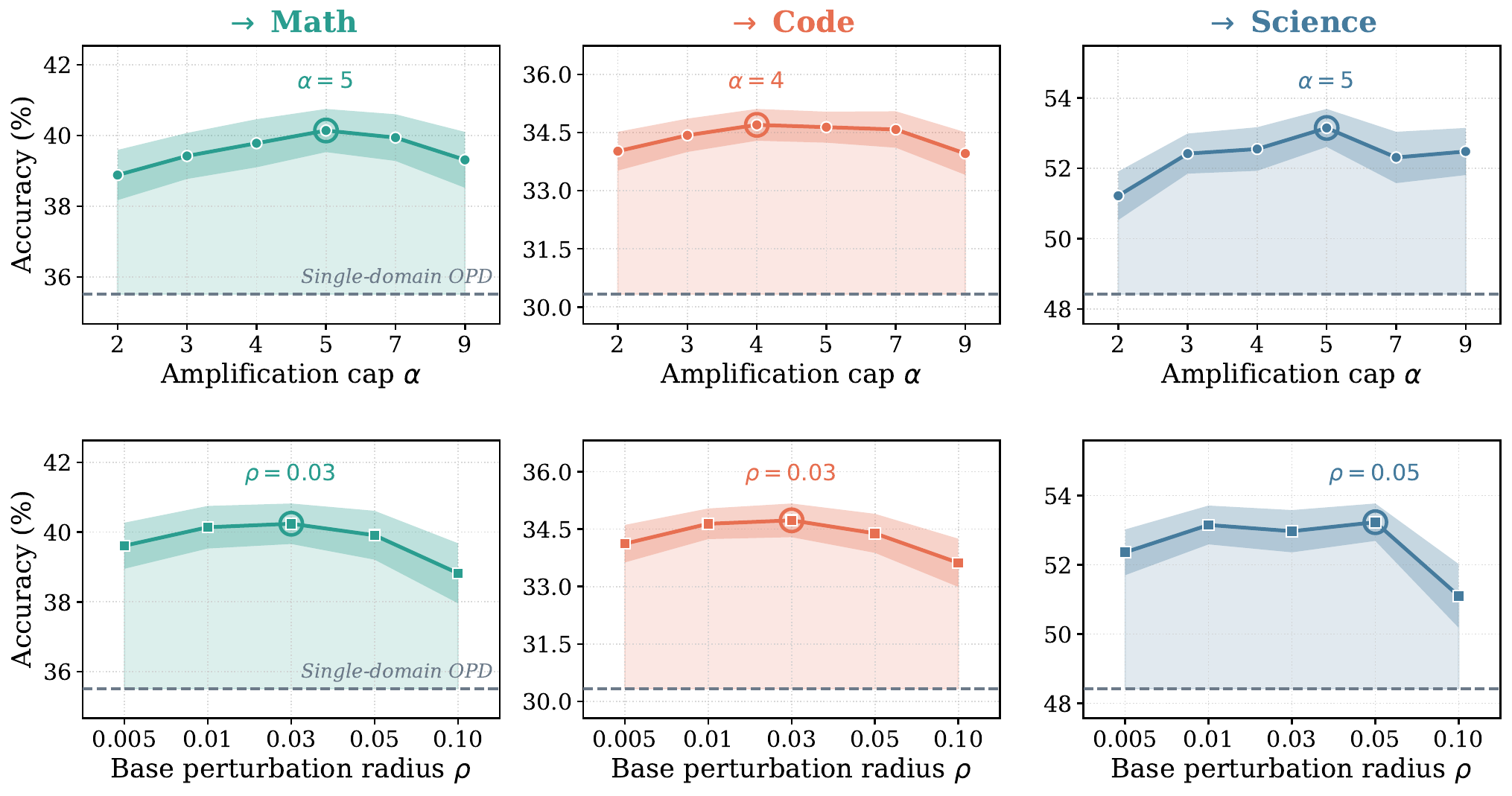}
\caption{\textbf{Sweeps over the two hyperparameters of \ourmodel on the Qwen3-1.7B students.}
The top row varies the amplification cap $\alpha$ of~\Cref{eq:scale} and the bottom row varies the
base radius $\rho$ of~\Cref{eq:objective}, each with the other held at its default. Each column
fixes a target domain and each point averages the two transfer directions of~\Cref{tab:main} that
share it, so either row covers the six directions once. The band is the standard deviation over
five runs, the dashed line is the Single-domain OPD reference of~\Cref{tab:main}, and the circled
marker is the peak of that curve.}
\label{fig:hparam}
\vspace{-0.5 em}
\end{figure*}

Because \Cref{eq:objective} constrains an update-salient coordinate over $\rho\,s_k(i)$ rather than
over $\rho$, and that semi-axis reaches $\rho\alpha$, one may ask whether \ourmodel merely trains
under a larger perturbation radius and whether an isotropic constraint of comparable size would do
as well. \Cref{fig:hparam} separates the two quantities that this reading conflates, namely the
overall size of the flatness budget and the way it is distributed across coordinates. The isotropic
constraint itself is the \emph{w/ Vanilla SAM} row of \Cref{tab:ablation}, which is $\alpha = 1$
with $\rho$ tuned over the grid of the bottom row, and it is the level against which the top row
has to be read.

The two rows behave differently. Varying $\alpha$ moves the average over the six directions by
$1.27$ points and every setting of the grid stays above the isotropic result on all three targets,
so what the objective is sensitive to is the profile of the budget across coordinates rather than
the presence of a constraint. Each curve rises to a peak and then turns down, which is the behaviour \Cref{eq:ratio}
leads one to expect: the displacement factor can only be reduced by redistributing the budget, and
past a certain amplification the curvature factor grows faster than the displacement factor falls,
which is why the construction caps the amplification instead of leaving it unbounded. Varying
$\rho$ leaves the average within $0.62$ points over the tenfold interval from $0.005$ to $0.05$ and
keeps \ourmodel above the single-domain reference at every setting, with the gain narrowing only at
the loosest radius of the grid, since a constraint that loose asks the network to
keep the loss low over a neighbourhood in which the quadratic picture underlying
\Cref{eq:lemma} no longer holds. This asymmetry is what \Cref{prop:bound} predicts, since its
right-hand side is invariant under $S \mapsto cS$ so that only the relative profile of the scales
can affect it, while $\rho$ enters \Cref{eq:lemma} through a multiplicative $\rho^2$ that sets how
hard the curvature is suppressed rather than along which directions.

The comparison an isotropic constraint of matched size would provide is already contained in
\Cref{tab:ablation}, whose \emph{w/ Vanilla SAM} row was tuned over the grid of the bottom row and
sits $1.75$ points below \ourmodel. A single pair, $\alpha = 5$ and $\rho = 0.03$, is used in every
other experiment of this paper, unchanged across the three domains, the six transfer directions and
both student scales, and we keep it rather than fitting a pair to each domain. What the two rows establish
is in any case not the merit of one setting over another but the insensitivity of the conclusion to
the choice, since every setting of either grid leaves \ourmodel above the reference that row is read
against and the least favourable one does so as well. The comparisons of \Cref{tab:main} therefore
do not rest on where inside these ranges the pair is placed. The warm-up length $N$ and the
reference fraction $p$ are held at the values of \Cref{subsec:perturbation} throughout.

\subsection{Compatibility with Downstream Merging Operators.}\label{app:merger}

\begin{figure*}[h]
\centering
\includegraphics[width=1.0\linewidth]{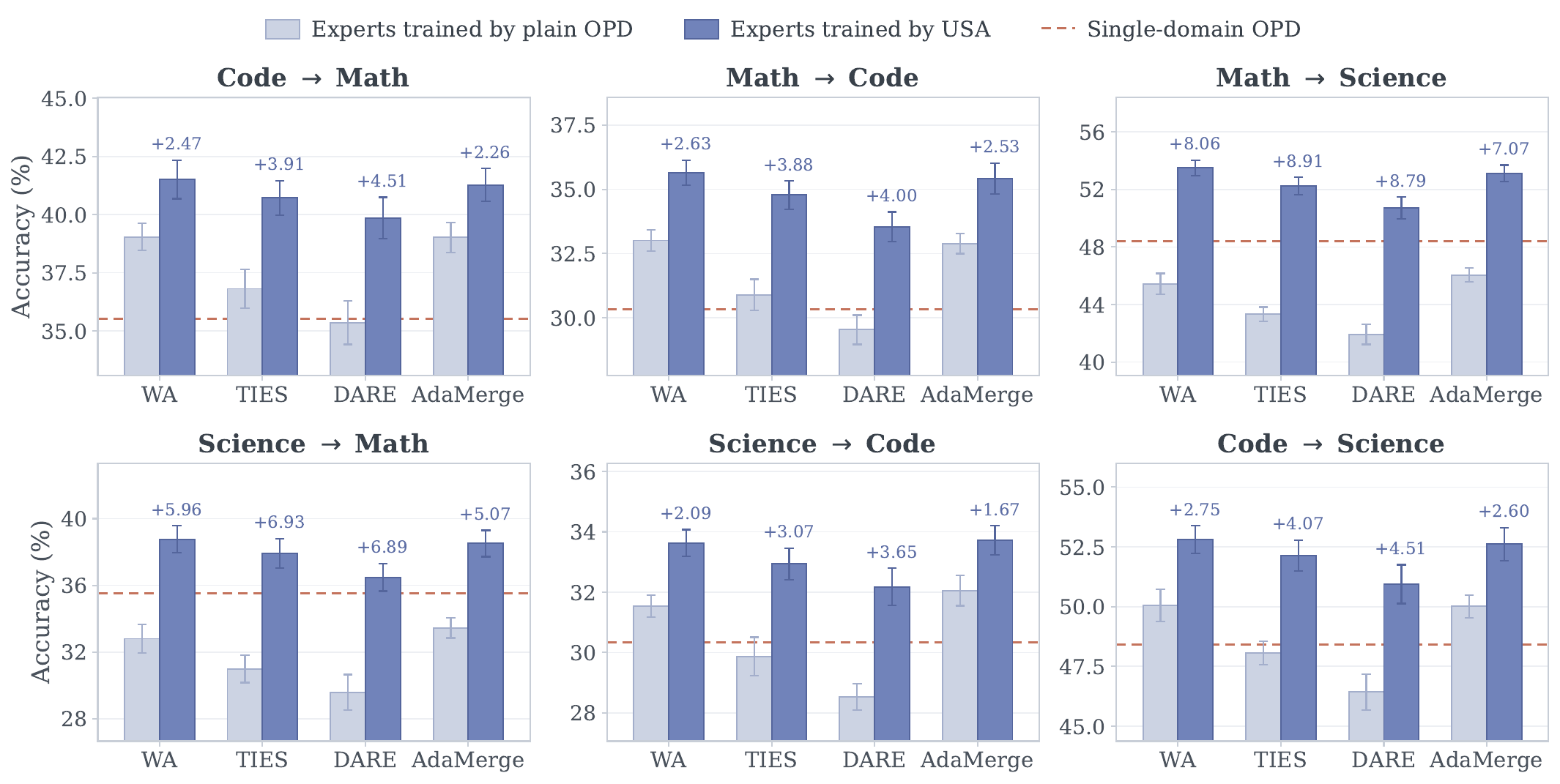}
\caption{\textbf{Swapping the downstream merging operator, on the Qwen3-1.7B students.}
Each panel is one transfer direction and each group of bars one operator, with the light bar
merging experts trained by plain OPD and the dark bar merging experts trained by \ourmodel.
The two panels of a column share a target domain and therefore also the dashed Single-domain
OPD reference of~\Cref{tab:main}. Labels give the gain of
\ourmodel over plain OPD under the same operator, and error bars are standard deviations over
five runs.}
\label{fig:merger}
\vspace{-0.5 em}
\end{figure*}

\ourmodel intervenes on how each domain expert is trained and leaves the merging step
untouched, so the gains reported in~\Cref{sec:experiment} could in principle be specific to the
operator those experiments use. To test this we keep the experts fixed and swap the operator for
three further choices beyond weight averaging, namely TIES~\citep{ties},
DARE~\citep{dare} applied before weight averaging, and
AdaMerging~\citep{adamerging}, and we run the swap twice, once on experts trained by
plain OPD and once on experts trained by \ourmodel, which separates what the operator
contributes from what the training contributes.

\begin{itemize}[leftmargin=*]

\item \textbf{Obs 11: the gain of \ourmodel survives every operator we try, whereas plain OPD
depends on the operator to stay competitive at all.} Merging experts trained by \ourmodel
improves on merging experts trained by plain OPD under all four operators and in all six
directions, and it clears the single-domain reference in every one of those cells. Plain OPD
does not: with DARE it falls below that reference in all six directions and with TIES in four
of them, so under plain OPD the merged model can be worse than not merging at all, and which
operator is used decides whether the cross-domain step is worth performing. The spread across
the four operators also narrows from $3.71$ to $2.04$ points on average, and the weakest
operator under \ourmodel still finishes above the strongest operator under plain OPD. Weight
averaging remains the best or joint-best of the four on almost every direction, which is why
the main tables use it, so the conclusion is not that the choice of operator ceases to matter
but that it stops deciding whether cross-domain merging helps or hurts.

\item \textbf{Obs 12: the operators lose ground for two reasons, neither of which is a defect
of the operators themselves.} They were designed to consolidate several tasks into a single
checkpoint while preserving all of them, and weight averaging in particular was introduced to
recover value from runs that would otherwise be discarded, so none of them is built to spend a
source domain in order to buy accuracy on a target. Their parameter-level operations also
presuppose the task vectors that supervised fine-tuning produces. DARE is the clearest
instance, since its random dropping and rescaling is justified by the extreme redundancy of
those deltas, and the updates that on-policy distillation produces are sparse and heavy-tailed
in the sense of~\Cref{subsec:locating}, so dropping at a rate calibrated for redundant deltas
removes coordinates that carry the update instead. This is consistent with DARE being the
weakest of the four under plain OPD and with it being the operator that \ourmodel helps most,
by $5.39$ points on average, since concentrating the flatness budget on the update-salient
coordinates is what makes those coordinates tolerant to the perturbation that the operator
applies to them.

\end{itemize}

\subsection{Generalization to Search Agents.}\label{app:search}

\begin{figure*}[h]
\centering
\includegraphics[width=1.0\linewidth]{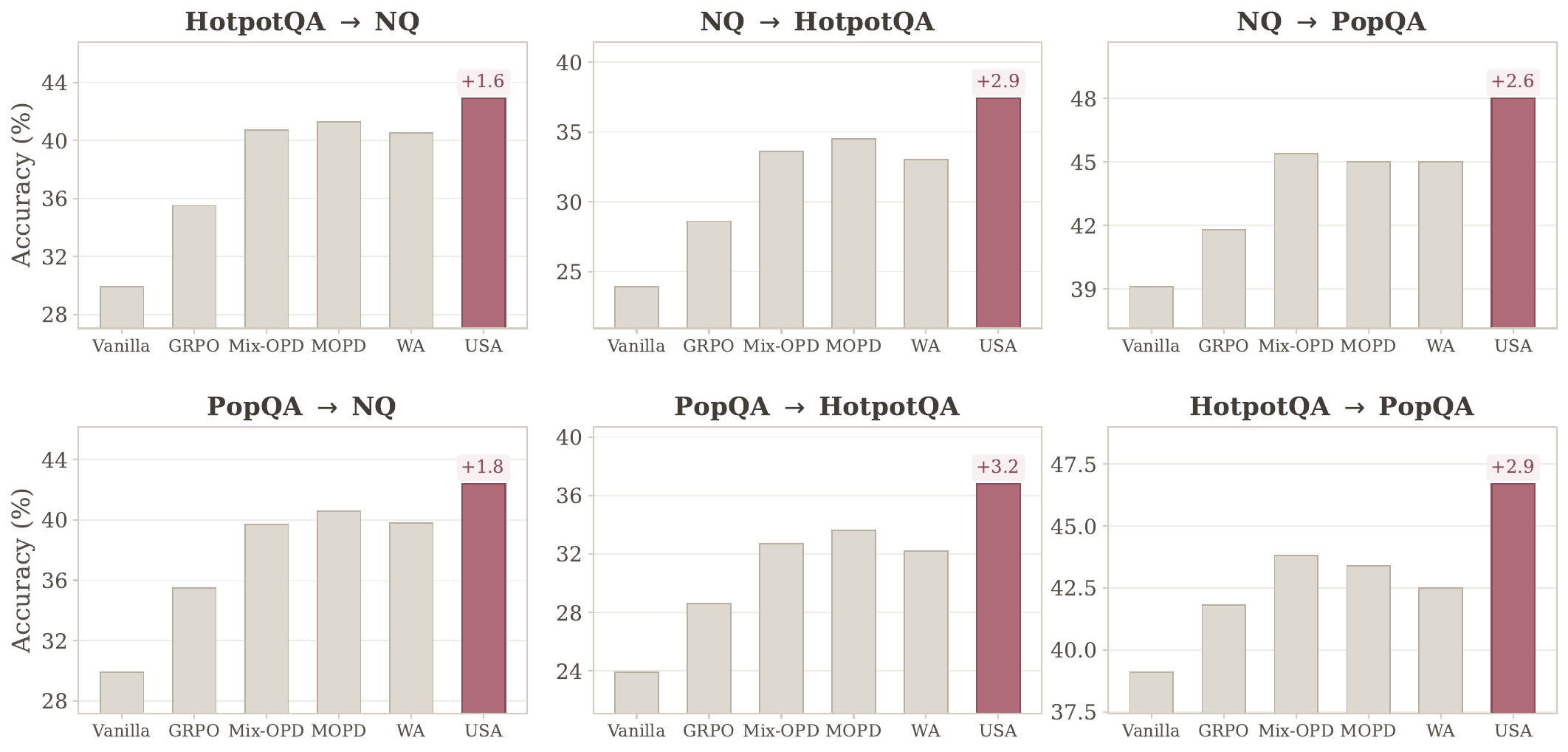}
\caption{\textbf{Transfer between search-agent domains on the Qwen3-1.7B students.}
Each panel is one transfer direction and each bar one method, with \ourmodel drawn in colour and
the five baselines in a neutral tone. Labels give the margin of \ourmodel over
the strongest baseline of that panel. Every configuration is run once, so no deviations are
shown.}
\label{fig:search}
\vspace{-0.5 em}
\end{figure*}

To examine whether the benefit of \ourmodel is specific to code-interpreter agents, we further
evaluate it in a retrieval-based search-agent environment. We follow the interaction protocol of
Search-R1~\citep{jin2025searchr1}, where the model alternates between reasoning and search actions,
and retrieved passages are appended to the trajectory as environment observations. We instantiate
three search-task domains with distinct retrieval and reasoning characteristics: Natural Questions
(NQ)~\citep{kwiatkowski2019natural} for general open-domain question answering,
HotpotQA~\citep{yang2018hotpotqa} for multi-hop question answering, and
PopQA~\citep{mallen2023trust} for long-tail factual question answering. We uniformly sample
$8$K training questions from each domain, yielding $24$K questions in total.

The search environment follows Search-R1~\citep{jin2025searchr1}. We use a fixed Wikipedia
corpus with an E5 dense retriever~\citep{wang2022e5}, return the top three passages for each
search action, and allow at most four search actions per trajectory. Retrieved passages are inserted
into the context but are excluded from the training loss, exactly as environment observations are
handled in~\Cref{eq:opd}. All domains share the same search environment and differ only in
their training distributions.

We use Qwen3-8B as the teacher and Qwen3-1.7B as the student. Each domain teacher is first
optimized with GRPO on its corresponding $8$K subset, after which the student is trained with OPD
or \ourmodel following the same pipeline as in the main experiments. We evaluate all six directed
source--target pairs formed by the three domains. Weight Average and \ourmodel merge independently
trained source and target experts with equal coefficients, $\lambda_s=\lambda_t=0.5$.
Mix-data OPD and MOPD follow the same definitions as in our main experiments, and Vanilla and
GRPO are reported alongside them. Importantly, we reuse the \ourmodel
hyperparameters selected in the code-interpreter experiments without search-specific tuning:
$N=10$, $p=1\%$, $\rho=0.03$, and $\alpha=5$. We report Accuracy (\%), where a prediction is
counted as correct when its normalized final answer exactly matches the gold answer. Each
configuration is run once, and we therefore report point estimates without standard deviations.

\begin{itemize}[leftmargin=*]

\item \textbf{Obs 13: the ordering of the methods carries over to an environment that shares
nothing with the main experiments but its interaction pattern.} \ourmodel is the strongest method
on every one of the six directions, ahead of the best baseline of each direction by $2.50$ points
on average and by at least $1.6$ points everywhere, and it holds that position against MOPD and
Mix-data OPD, which train on both domains jointly, as well as against weight averaging, which
merges the same two experts it does. The margin is of the same order as the one
of~\Cref{tab:main} even though the teacher, the student's training data, the tool and the metric
have all changed, so what the update-aware perturbation geometry exploits is the structure of the
updates that on-policy distillation produces rather than anything particular to executing code.

\item \textbf{Obs 14: the transfer requires no environment-specific tuning.} The amplification cap,
the reference fraction, the warm-up length and the base radius are all carried over unchanged from
the code-interpreter study of~\Cref{app:hparam}, even though the teacher, the student's data, the
tool and the evaluation metric all differ. This is what the construction of~\Cref{eq:scale}
predicts, since the reference level is a within-matrix quantile of the update magnitudes and
therefore adapts to whatever scale the updates happen to take in a new environment, leaving
$\alpha$ to set only the ratio between the largest and the smallest semi-axis. The quantities that
would need retuning if the mechanism were tied to the code-interpreter setting are exactly the ones
that did not.

\end{itemize}

\subsection{Per-Benchmark Results.}\label{app:detailed}

\Cref{tab:main} reports, for each transfer direction, the mean over the two benchmarks of the
target domain. \Cref{tab:detailed_1p7b,tab:detailed_4b} give the underlying per-benchmark scores
for the Qwen3-1.7B and Qwen3-4B students respectively, with the rows and the transfer directions
arranged as in the main table so that the two can be read side by side.

\begin{table*}[t]
    \centering
    \caption{
    \textbf{Per-benchmark results for the Qwen3-1.7B students.}
    Each pair of columns expands one cell of~\Cref{tab:main} into the two benchmarks of the
    target domain. The best results are in \textbf{bold}, and the second-best results are
    \underline{underlined}. We report average@32 over 5 runs.
    }
    \vspace{1mm}
    \label{tab:detailed_1p7b}

    \begingroup
    \setlength{\tabcolsep}{4.5pt}
    \renewcommand{\arraystretch}{1.08}
    \normalsize

    \begin{adjustbox}{max width=0.98\textwidth}
    \begin{tabular}{lcccccc}

        \toprule

        \multirow{2}{*}{\small\textbf{Method}}
        & \multicolumn{2}{c}{\small\textbf{Code $\rightarrow$ Math}}
        & \multicolumn{2}{c}{\small\textbf{Math $\rightarrow$ Code}}
        & \multicolumn{2}{c}{\small\textbf{Math $\rightarrow$ Sci.}} \\

        \cmidrule(lr){2-3}
        \cmidrule(lr){4-5}
        \cmidrule(lr){6-7}

        &
        {\footnotesize AIME 2025}
        & {\footnotesize HMMT Feb.\ 2026}
        & {\footnotesize LiveCodeBench-v6}
        & {\footnotesize NaturalCodeBench}
        & {\footnotesize GPQA-Diamond}
        & {\footnotesize SciBench-Atkins} \\

        \midrule

        \rowcolor[HTML]{FEE090}
        \multicolumn{7}{c}{
            \textbf{Teacher Model from the Qwen3-14B Series}
        } \\

        GRPO
        & 68.38\phantom{.}\scalebox{0.72}{$\pm$1.42}
        & 48.91\phantom{.}\scalebox{0.72}{$\pm$0.87}
        & 69.24\phantom{.}\scalebox{0.72}{$\pm$0.93}
        & 54.57\phantom{.}\scalebox{0.72}{$\pm$0.71}
        & 60.53\phantom{.}\scalebox{0.72}{$\pm$1.06}
        & 81.47\phantom{.}\scalebox{0.72}{$\pm$0.84} \\

        \midrule

        \rowcolor[HTML]{E0F3F8}
        \multicolumn{7}{c}{
            \textbf{Student Models from the Qwen3-1.7B Series}
        } \\

        Vanilla
        & 8.96\phantom{.}\scalebox{0.72}{$\pm$0.94}
        & 21.21\phantom{.}\scalebox{0.72}{$\pm$1.07}
        & 22.73\phantom{.}\scalebox{0.72}{$\pm$0.74}
        & 15.33\phantom{.}\scalebox{0.72}{$\pm$0.55}
        & 26.80\phantom{.}\scalebox{0.72}{$\pm$0.69}
        & 46.01\phantom{.}\scalebox{0.72}{$\pm$0.88} \\

        Mix-data SFT
        & 27.41\phantom{.}\scalebox{0.72}{$\pm$0.88}
        & 24.58\phantom{.}\scalebox{0.72}{$\pm$1.69}
        & 25.91\phantom{.}\scalebox{0.72}{$\pm$0.56}
        & 18.72\phantom{.}\scalebox{0.72}{$\pm$0.74}
        & 29.63\phantom{.}\scalebox{0.72}{$\pm$0.61}
        & 53.16\phantom{.}\scalebox{0.72}{$\pm$1.15} \\

        Mix-data GRPO
        & 35.88\phantom{.}\scalebox{0.72}{$\pm$1.41}
        & 29.22\phantom{.}\scalebox{0.72}{$\pm$0.79}
        & 31.58\phantom{.}\scalebox{0.72}{$\pm$0.91}
        & 22.06\phantom{.}\scalebox{0.72}{$\pm$0.48}
        & 31.48\phantom{.}\scalebox{0.72}{$\pm$0.84}
        & 58.72\phantom{.}\scalebox{0.72}{$\pm$0.67} \\

        Mix-data OPD
        & 42.21\phantom{.}\scalebox{0.72}{$\pm$1.06}
        & 33.61\phantom{.}\scalebox{0.72}{$\pm$1.32}
        & 37.86\phantom{.}\scalebox{0.72}{$\pm$0.68}
        & 26.37\phantom{.}\scalebox{0.72}{$\pm$0.59}
        & 33.07\phantom{.}\scalebox{0.72}{$\pm$0.52}
        & 60.84\phantom{.}\scalebox{0.72}{$\pm$1.09} \\

        Single-domain OPD
        & 39.58\phantom{.}\scalebox{0.72}{$\pm$1.21}
        & 31.44\phantom{.}\scalebox{0.72}{$\pm$0.96}
        & 35.73\phantom{.}\scalebox{0.72}{$\pm$0.87}
        & 24.92\phantom{.}\scalebox{0.72}{$\pm$0.71}
        & \underline{34.48}\phantom{.}\scalebox{0.72}{$\pm$0.76}
        & 62.36\phantom{.}\scalebox{0.72}{$\pm$1.03} \\

        MOPD
        & 43.84\phantom{.}\scalebox{0.72}{$\pm$0.82}
        & 34.19\phantom{.}\scalebox{0.72}{$\pm$1.18}
        & \underline{39.18}\phantom{.}\scalebox{0.72}{$\pm$0.47}
        & 27.05\phantom{.}\scalebox{0.72}{$\pm$0.83}
        & 34.12\phantom{.}\scalebox{0.72}{$\pm$0.66}
        & \underline{63.58}\phantom{.}\scalebox{0.72}{$\pm$0.78} \\

        \cmidrule(lr){1-7}

        Weight Average
        & \underline{44.12}\phantom{.}\scalebox{0.72}{$\pm$0.44}
        & 33.96\phantom{.}\scalebox{0.72}{$\pm$1.07}
        & 38.74\phantom{.}\scalebox{0.72}{$\pm$0.75}
        & \underline{27.28}\phantom{.}\scalebox{0.72}{$\pm$0.39}
        & 31.67\phantom{.}\scalebox{0.72}{$\pm$0.91}
        & 59.21\phantom{.}\scalebox{0.72}{$\pm$1.12} \\

        Ties-Merging
        & 41.47\phantom{.}\scalebox{0.72}{$\pm$1.52}
        & 32.14\phantom{.}\scalebox{0.72}{$\pm$0.73}
        & 36.44\phantom{.}\scalebox{0.72}{$\pm$1.03}
        & 25.36\phantom{.}\scalebox{0.72}{$\pm$0.62}
        & 29.82\phantom{.}\scalebox{0.72}{$\pm$0.43}
        & 56.84\phantom{.}\scalebox{0.72}{$\pm$0.89} \\

        DARE+WA
        & 39.82\phantom{.}\scalebox{0.72}{$\pm$1.17}
        & 30.88\phantom{.}\scalebox{0.72}{$\pm$1.45}
        & 34.97\phantom{.}\scalebox{0.72}{$\pm$0.84}
        & 24.11\phantom{.}\scalebox{0.72}{$\pm$0.76}
        & 28.54\phantom{.}\scalebox{0.72}{$\pm$0.58}
        & 55.31\phantom{.}\scalebox{0.72}{$\pm$1.28} \\

        AdaMerging
        & 43.67\phantom{.}\scalebox{0.72}{$\pm$0.97}
        & \underline{34.37}\phantom{.}\scalebox{0.72}{$\pm$0.88}
        & 39.06\phantom{.}\scalebox{0.72}{$\pm$0.61}
        & 26.71\phantom{.}\scalebox{0.72}{$\pm$0.50}
        & 32.35\phantom{.}\scalebox{0.72}{$\pm$0.72}
        & 59.74\phantom{.}\scalebox{0.72}{$\pm$0.65} \\

        \cmidrule(lr){1-7}

        \textbf{\ourmodel}
        & \textbf{46.90}\phantom{.}\scalebox{0.72}{$\pm$1.29}
        & \textbf{36.12}\phantom{.}\scalebox{0.72}{$\pm$1.04}
        & \textbf{42.31}\phantom{.}\scalebox{0.72}{$\pm$0.80}
        & \textbf{28.97}\phantom{.}\scalebox{0.72}{$\pm$0.57}
        & \textbf{37.82}\phantom{.}\scalebox{0.72}{$\pm$0.47}
        & \textbf{69.18}\phantom{.}\scalebox{0.72}{$\pm$0.94} \\

        \midrule
        \midrule

        \multirow{2}{*}{\small\textbf{Method}}
        & \multicolumn{2}{c}{\small\textbf{Sci. $\rightarrow$ Math}}
        & \multicolumn{2}{c}{\small\textbf{Sci. $\rightarrow$ Code}}
        & \multicolumn{2}{c}{\small\textbf{Code $\rightarrow$ Sci.}} \\

        \cmidrule(lr){2-3}
        \cmidrule(lr){4-5}
        \cmidrule(lr){6-7}

        &
        {\footnotesize AIME 2025}
        & {\footnotesize HMMT Feb.\ 2026}
        & {\footnotesize LiveCodeBench-v6}
        & {\footnotesize NaturalCodeBench}
        & {\footnotesize GPQA-Diamond}
        & {\footnotesize SciBench-Atkins} \\

        \midrule

        \rowcolor[HTML]{FEE090}
        \multicolumn{7}{c}{
            \textbf{Teacher Model from the Qwen3-14B Series}
        } \\

        GRPO
        & 68.38\phantom{.}\scalebox{0.72}{$\pm$1.42}
        & 48.91\phantom{.}\scalebox{0.72}{$\pm$0.87}
        & 69.24\phantom{.}\scalebox{0.72}{$\pm$0.93}
        & 54.57\phantom{.}\scalebox{0.72}{$\pm$0.71}
        & 60.53\phantom{.}\scalebox{0.72}{$\pm$1.06}
        & 81.47\phantom{.}\scalebox{0.72}{$\pm$0.84} \\

        \midrule

        \rowcolor[HTML]{E0F3F8}
        \multicolumn{7}{c}{
            \textbf{Student Models from the Qwen3-1.7B Series}
        } \\

        Vanilla
        & 8.96\phantom{.}\scalebox{0.72}{$\pm$0.94}
        & 21.21\phantom{.}\scalebox{0.72}{$\pm$1.07}
        & 22.73\phantom{.}\scalebox{0.72}{$\pm$0.74}
        & 15.33\phantom{.}\scalebox{0.72}{$\pm$0.55}
        & 26.80\phantom{.}\scalebox{0.72}{$\pm$0.69}
        & 46.01\phantom{.}\scalebox{0.72}{$\pm$0.88} \\

        Mix-data SFT
        & 26.23\phantom{.}\scalebox{0.72}{$\pm$1.46}
        & 23.91\phantom{.}\scalebox{0.72}{$\pm$0.81}
        & 25.36\phantom{.}\scalebox{0.72}{$\pm$0.69}
        & 19.03\phantom{.}\scalebox{0.72}{$\pm$0.43}
        & 30.58\phantom{.}\scalebox{0.72}{$\pm$0.57}
        & 54.39\phantom{.}\scalebox{0.72}{$\pm$1.14} \\

        Mix-data GRPO
        & 32.71\phantom{.}\scalebox{0.72}{$\pm$0.93}
        & 27.84\phantom{.}\scalebox{0.72}{$\pm$1.38}
        & 31.21\phantom{.}\scalebox{0.72}{$\pm$0.52}
        & 21.54\phantom{.}\scalebox{0.72}{$\pm$0.78}
        & 32.82\phantom{.}\scalebox{0.72}{$\pm$0.88}
        & 59.46\phantom{.}\scalebox{0.72}{$\pm$0.70} \\

        Mix-data OPD
        & 37.29\phantom{.}\scalebox{0.72}{$\pm$1.57}
        & 30.18\phantom{.}\scalebox{0.72}{$\pm$0.74}
        & 36.65\phantom{.}\scalebox{0.72}{$\pm$0.95}
        & 25.69\phantom{.}\scalebox{0.72}{$\pm$0.56}
        & 35.39\phantom{.}\scalebox{0.72}{$\pm$0.49}
        & 63.02\phantom{.}\scalebox{0.72}{$\pm$1.06} \\

        Single-domain OPD
        & 39.58\phantom{.}\scalebox{0.72}{$\pm$1.21}
        & \underline{31.44}\phantom{.}\scalebox{0.72}{$\pm$0.96}
        & 35.73\phantom{.}\scalebox{0.72}{$\pm$0.87}
        & 24.92\phantom{.}\scalebox{0.72}{$\pm$0.71}
        & 34.48\phantom{.}\scalebox{0.72}{$\pm$0.76}
        & 62.36\phantom{.}\scalebox{0.72}{$\pm$1.03} \\

        MOPD
        & \underline{40.31}\phantom{.}\scalebox{0.72}{$\pm$1.03}
        & 30.96\phantom{.}\scalebox{0.72}{$\pm$1.21}
        & 37.83\phantom{.}\scalebox{0.72}{$\pm$0.64}
        & \underline{26.31}\phantom{.}\scalebox{0.72}{$\pm$0.84}
        & 36.14\phantom{.}\scalebox{0.72}{$\pm$0.95}
        & \underline{64.27}\phantom{.}\scalebox{0.72}{$\pm$0.62} \\

        \cmidrule(lr){1-7}

        Weight Average
        & 36.42\phantom{.}\scalebox{0.72}{$\pm$0.86}
        & 29.18\phantom{.}\scalebox{0.72}{$\pm$1.49}
        & 37.19\phantom{.}\scalebox{0.72}{$\pm$0.58}
        & 25.88\phantom{.}\scalebox{0.72}{$\pm$0.45}
        & \underline{36.39}\phantom{.}\scalebox{0.72}{$\pm$0.68}
        & 63.71\phantom{.}\scalebox{0.72}{$\pm$1.18} \\

        Ties-Merging
        & 34.31\phantom{.}\scalebox{0.72}{$\pm$1.35}
        & 27.66\phantom{.}\scalebox{0.72}{$\pm$0.92}
        & 35.26\phantom{.}\scalebox{0.72}{$\pm$1.08}
        & 24.47\phantom{.}\scalebox{0.72}{$\pm$0.67}
        & 34.62\phantom{.}\scalebox{0.72}{$\pm$0.55}
        & 61.49\phantom{.}\scalebox{0.72}{$\pm$0.81} \\

        DARE+WA
        & 32.88\phantom{.}\scalebox{0.72}{$\pm$1.82}
        & 26.29\phantom{.}\scalebox{0.72}{$\pm$1.10}
        & 33.84\phantom{.}\scalebox{0.72}{$\pm$0.77}
        & 23.22\phantom{.}\scalebox{0.72}{$\pm$0.37}
        & 32.97\phantom{.}\scalebox{0.72}{$\pm$0.83}
        & 59.88\phantom{.}\scalebox{0.72}{$\pm$1.26} \\

        AdaMerging
        & 37.16\phantom{.}\scalebox{0.72}{$\pm$1.13}
        & 29.72\phantom{.}\scalebox{0.72}{$\pm$0.43}
        & \underline{38.02}\phantom{.}\scalebox{0.72}{$\pm$0.46}
        & 26.08\phantom{.}\scalebox{0.72}{$\pm$0.89}
        & 35.96\phantom{.}\scalebox{0.72}{$\pm$0.60}
        & 64.06\phantom{.}\scalebox{0.72}{$\pm$0.75} \\

        \cmidrule(lr){1-7}

        \textbf{\ourmodel}
        & \textbf{43.25}\phantom{.}\scalebox{0.72}{$\pm$1.00}
        & \textbf{34.27}\phantom{.}\scalebox{0.72}{$\pm$1.27}
        & \textbf{39.65}\phantom{.}\scalebox{0.72}{$\pm$0.72}
        & \textbf{27.60}\phantom{.}\scalebox{0.72}{$\pm$0.51}
        & \textbf{38.43}\phantom{.}\scalebox{0.72}{$\pm$0.74}
        & \textbf{67.16}\phantom{.}\scalebox{0.72}{$\pm$0.90} \\

        \bottomrule

    \end{tabular}
    \end{adjustbox}

    \endgroup

    \vspace{-10pt}
\end{table*}

\begin{table*}[t]
    \centering
    \caption{
    \textbf{Per-benchmark results for the Qwen3-4B students.}
    Each pair of columns expands one cell of~\Cref{tab:main} into the two benchmarks of the
    target domain. The best results are in \textbf{bold}, and the second-best results are
    \underline{underlined}. We report average@32 over 5 runs.
    }
    \vspace{1mm}
    \label{tab:detailed_4b}

    \begingroup
    \setlength{\tabcolsep}{4.5pt}
    \renewcommand{\arraystretch}{1.08}
    \normalsize

    \begin{adjustbox}{max width=0.98\textwidth}
    \begin{tabular}{lcccccc}

        \toprule

        \multirow{2}{*}{\small\textbf{Method}}
        & \multicolumn{2}{c}{\small\textbf{Code $\rightarrow$ Math}}
        & \multicolumn{2}{c}{\small\textbf{Math $\rightarrow$ Code}}
        & \multicolumn{2}{c}{\small\textbf{Math $\rightarrow$ Sci.}} \\

        \cmidrule(lr){2-3}
        \cmidrule(lr){4-5}
        \cmidrule(lr){6-7}

        &
        {\footnotesize AIME 2025}
        & {\footnotesize HMMT Feb.\ 2026}
        & {\footnotesize LiveCodeBench-v6}
        & {\footnotesize NaturalCodeBench}
        & {\footnotesize GPQA-Diamond}
        & {\footnotesize SciBench-Atkins} \\

        \midrule

        \rowcolor[HTML]{FEE090}
        \multicolumn{7}{c}{
            \textbf{Teacher Model from the Qwen3-14B Series}
        } \\

        GRPO
        & 68.38\phantom{.}\scalebox{0.72}{$\pm$1.42}
        & 48.91\phantom{.}\scalebox{0.72}{$\pm$0.87}
        & 69.24\phantom{.}\scalebox{0.72}{$\pm$0.93}
        & 54.57\phantom{.}\scalebox{0.72}{$\pm$0.71}
        & 60.53\phantom{.}\scalebox{0.72}{$\pm$1.06}
        & 81.47\phantom{.}\scalebox{0.72}{$\pm$0.84} \\

        \midrule

        \rowcolor[HTML]{E0F3F8}
        \multicolumn{7}{c}{
            \textbf{Student Models from the Qwen3-4B Series}
        } \\

        Vanilla
        & 46.73\phantom{.}\scalebox{0.72}{$\pm$1.12}
        & 26.84\phantom{.}\scalebox{0.72}{$\pm$1.16}
        & 45.61\phantom{.}\scalebox{0.72}{$\pm$0.72}
        & 33.47\phantom{.}\scalebox{0.72}{$\pm$0.68}
        & 44.38\phantom{.}\scalebox{0.72}{$\pm$0.87}
        & 58.62\phantom{.}\scalebox{0.72}{$\pm$0.84} \\

        Mix-data SFT
        & 52.34\phantom{.}\scalebox{0.72}{$\pm$1.08}
        & 31.08\phantom{.}\scalebox{0.72}{$\pm$1.37}
        & 53.21\phantom{.}\scalebox{0.72}{$\pm$0.81}
        & 37.08\phantom{.}\scalebox{0.72}{$\pm$0.54}
        & 47.62\phantom{.}\scalebox{0.72}{$\pm$0.63}
        & 62.51\phantom{.}\scalebox{0.72}{$\pm$1.17} \\

        Mix-data GRPO
        & 58.76\phantom{.}\scalebox{0.72}{$\pm$1.46}
        & 38.10\phantom{.}\scalebox{0.72}{$\pm$0.86}
        & 60.31\phantom{.}\scalebox{0.72}{$\pm$0.58}
        & 42.18\phantom{.}\scalebox{0.72}{$\pm$0.73}
        & 53.21\phantom{.}\scalebox{0.72}{$\pm$0.89}
        & 68.44\phantom{.}\scalebox{0.72}{$\pm$0.69} \\

        Mix-data OPD
        & 59.42\phantom{.}\scalebox{0.72}{$\pm$1.22}
        & 38.08\phantom{.}\scalebox{0.72}{$\pm$1.49}
        & 61.04\phantom{.}\scalebox{0.72}{$\pm$0.72}
        & 42.93\phantom{.}\scalebox{0.72}{$\pm$0.47}
        & 52.84\phantom{.}\scalebox{0.72}{$\pm$0.55}
        & 68.91\phantom{.}\scalebox{0.72}{$\pm$1.06} \\

        Single-domain OPD
        & 58.71\phantom{.}\scalebox{0.72}{$\pm$1.31}
        & 37.56\phantom{.}\scalebox{0.72}{$\pm$1.09}
        & 60.22\phantom{.}\scalebox{0.72}{$\pm$0.89}
        & 42.31\phantom{.}\scalebox{0.72}{$\pm$0.86}
        & \underline{54.16}\phantom{.}\scalebox{0.72}{$\pm$0.98}
        & \underline{69.26}\phantom{.}\scalebox{0.72}{$\pm$1.03} \\

        MOPD
        & \underline{60.55}\phantom{.}\scalebox{0.72}{$\pm$0.91}
        & \underline{38.61}\phantom{.}\scalebox{0.72}{$\pm$1.13}
        & \underline{61.74}\phantom{.}\scalebox{0.72}{$\pm$0.94}
        & 43.62\phantom{.}\scalebox{0.72}{$\pm$0.66}
        & 53.72\phantom{.}\scalebox{0.72}{$\pm$0.74}
        & 68.83\phantom{.}\scalebox{0.72}{$\pm$0.83} \\

        \cmidrule(lr){1-7}

        Weight Average
        & 60.21\phantom{.}\scalebox{0.72}{$\pm$1.57}
        & 38.42\phantom{.}\scalebox{0.72}{$\pm$0.98}
        & 61.58\phantom{.}\scalebox{0.72}{$\pm$0.63}
        & 43.47\phantom{.}\scalebox{0.72}{$\pm$0.81}
        & 51.48\phantom{.}\scalebox{0.72}{$\pm$0.46}
        & 66.55\phantom{.}\scalebox{0.72}{$\pm$1.11} \\

        Ties-Merging
        & 58.46\phantom{.}\scalebox{0.72}{$\pm$1.19}
        & 37.06\phantom{.}\scalebox{0.72}{$\pm$1.42}
        & 59.45\phantom{.}\scalebox{0.72}{$\pm$0.85}
        & 42.87\phantom{.}\scalebox{0.72}{$\pm$0.59}
        & 49.87\phantom{.}\scalebox{0.72}{$\pm$0.97}
        & 64.28\phantom{.}\scalebox{0.72}{$\pm$0.72} \\

        DARE+WA
        & 57.11\phantom{.}\scalebox{0.72}{$\pm$1.68}
        & 35.88\phantom{.}\scalebox{0.72}{$\pm$0.79}
        & 58.22\phantom{.}\scalebox{0.72}{$\pm$1.03}
        & 40.63\phantom{.}\scalebox{0.72}{$\pm$0.70}
        & 50.14\phantom{.}\scalebox{0.72}{$\pm$0.61}
        & 62.91\phantom{.}\scalebox{0.72}{$\pm$1.23} \\

        AdaMerging
        & 60.48\phantom{.}\scalebox{0.72}{$\pm$1.01}
        & 38.13\phantom{.}\scalebox{0.72}{$\pm$0.92}
        & 61.33\phantom{.}\scalebox{0.72}{$\pm$0.49}
        & \underline{43.76}\phantom{.}\scalebox{0.72}{$\pm$0.88}
        & 51.93\phantom{.}\scalebox{0.72}{$\pm$0.82}
        & 67.02\phantom{.}\scalebox{0.72}{$\pm$0.65} \\

        \cmidrule(lr){1-7}

        \textbf{\ourmodel}
        & \textbf{63.93}\phantom{.}\scalebox{0.72}{$\pm$1.34}
        & \textbf{40.39}\phantom{.}\scalebox{0.72}{$\pm$1.25}
        & \textbf{66.25}\phantom{.}\scalebox{0.72}{$\pm$0.76}
        & \textbf{46.25}\phantom{.}\scalebox{0.72}{$\pm$0.52}
        & \textbf{58.98}\phantom{.}\scalebox{0.72}{$\pm$0.50}
        & \textbf{75.10}\phantom{.}\scalebox{0.72}{$\pm$0.95} \\

        \midrule
        \midrule

        \multirow{2}{*}{\small\textbf{Method}}
        & \multicolumn{2}{c}{\small\textbf{Sci. $\rightarrow$ Math}}
        & \multicolumn{2}{c}{\small\textbf{Sci. $\rightarrow$ Code}}
        & \multicolumn{2}{c}{\small\textbf{Code $\rightarrow$ Sci.}} \\

        \cmidrule(lr){2-3}
        \cmidrule(lr){4-5}
        \cmidrule(lr){6-7}

        &
        {\footnotesize AIME 2025}
        & {\footnotesize HMMT Feb.\ 2026}
        & {\footnotesize LiveCodeBench-v6}
        & {\footnotesize NaturalCodeBench}
        & {\footnotesize GPQA-Diamond}
        & {\footnotesize SciBench-Atkins} \\

        \midrule

        \rowcolor[HTML]{FEE090}
        \multicolumn{7}{c}{
            \textbf{Teacher Model from the Qwen3-14B Series}
        } \\

        GRPO
        & 68.38\phantom{.}\scalebox{0.72}{$\pm$1.42}
        & 48.91\phantom{.}\scalebox{0.72}{$\pm$0.87}
        & 69.24\phantom{.}\scalebox{0.72}{$\pm$0.93}
        & 54.57\phantom{.}\scalebox{0.72}{$\pm$0.71}
        & 60.53\phantom{.}\scalebox{0.72}{$\pm$1.06}
        & 81.47\phantom{.}\scalebox{0.72}{$\pm$0.84} \\

        \midrule

        \rowcolor[HTML]{E0F3F8}
        \multicolumn{7}{c}{
            \textbf{Student Models from the Qwen3-4B Series}
        } \\

        Vanilla
        & 46.73\phantom{.}\scalebox{0.72}{$\pm$1.12}
        & 26.84\phantom{.}\scalebox{0.72}{$\pm$1.16}
        & 45.61\phantom{.}\scalebox{0.72}{$\pm$0.72}
        & 33.47\phantom{.}\scalebox{0.72}{$\pm$0.68}
        & 44.38\phantom{.}\scalebox{0.72}{$\pm$0.87}
        & 58.62\phantom{.}\scalebox{0.72}{$\pm$0.84} \\

        Mix-data SFT
        & 50.81\phantom{.}\scalebox{0.72}{$\pm$1.53}
        & 30.46\phantom{.}\scalebox{0.72}{$\pm$0.82}
        & 52.94\phantom{.}\scalebox{0.72}{$\pm$0.62}
        & 36.71\phantom{.}\scalebox{0.72}{$\pm$0.76}
        & 48.22\phantom{.}\scalebox{0.72}{$\pm$0.78}
        & 63.08\phantom{.}\scalebox{0.72}{$\pm$0.84} \\

        Mix-data GRPO
        & 56.98\phantom{.}\scalebox{0.72}{$\pm$1.02}
        & 36.89\phantom{.}\scalebox{0.72}{$\pm$1.45}
        & 60.15\phantom{.}\scalebox{0.72}{$\pm$0.91}
        & 41.86\phantom{.}\scalebox{0.72}{$\pm$0.49}
        & 55.15\phantom{.}\scalebox{0.72}{$\pm$0.57}
        & 69.44\phantom{.}\scalebox{0.72}{$\pm$1.12} \\

        Mix-data OPD
        & 57.23\phantom{.}\scalebox{0.72}{$\pm$1.61}
        & 36.72\phantom{.}\scalebox{0.72}{$\pm$1.04}
        & 60.08\phantom{.}\scalebox{0.72}{$\pm$0.55}
        & 42.59\phantom{.}\scalebox{0.72}{$\pm$0.85}
        & 54.62\phantom{.}\scalebox{0.72}{$\pm$0.77}
        & 69.73\phantom{.}\scalebox{0.72}{$\pm$0.68} \\

        Single-domain OPD
        & 58.71\phantom{.}\scalebox{0.72}{$\pm$1.31}
        & 37.56\phantom{.}\scalebox{0.72}{$\pm$1.09}
        & 60.22\phantom{.}\scalebox{0.72}{$\pm$0.89}
        & 42.31\phantom{.}\scalebox{0.72}{$\pm$0.86}
        & 54.16\phantom{.}\scalebox{0.72}{$\pm$0.98}
        & 69.26\phantom{.}\scalebox{0.72}{$\pm$1.03} \\

        MOPD
        & \underline{59.33}\phantom{.}\scalebox{0.72}{$\pm$1.16}
        & \underline{37.92}\phantom{.}\scalebox{0.72}{$\pm$0.75}
        & 61.26\phantom{.}\scalebox{0.72}{$\pm$0.80}
        & \underline{43.18}\phantom{.}\scalebox{0.72}{$\pm$0.51}
        & 55.08\phantom{.}\scalebox{0.72}{$\pm$0.94}
        & 70.31\phantom{.}\scalebox{0.72}{$\pm$1.06} \\

        \cmidrule(lr){1-7}

        Weight Average
        & 56.27\phantom{.}\scalebox{0.72}{$\pm$0.87}
        & 35.91\phantom{.}\scalebox{0.72}{$\pm$1.36}
        & 60.94\phantom{.}\scalebox{0.72}{$\pm$0.77}
        & 42.88\phantom{.}\scalebox{0.72}{$\pm$0.42}
        & 54.91\phantom{.}\scalebox{0.72}{$\pm$0.69}
        & \underline{70.48}\phantom{.}\scalebox{0.72}{$\pm$0.79} \\

        Ties-Merging
        & 54.11\phantom{.}\scalebox{0.72}{$\pm$1.41}
        & 34.58\phantom{.}\scalebox{0.72}{$\pm$0.93}
        & 59.26\phantom{.}\scalebox{0.72}{$\pm$1.08}
        & 41.23\phantom{.}\scalebox{0.72}{$\pm$0.78}
        & 53.31\phantom{.}\scalebox{0.72}{$\pm$0.48}
        & 68.04\phantom{.}\scalebox{0.72}{$\pm$1.21} \\

        DARE+WA
        & 52.76\phantom{.}\scalebox{0.72}{$\pm$1.27}
        & 33.41\phantom{.}\scalebox{0.72}{$\pm$1.18}
        & 57.94\phantom{.}\scalebox{0.72}{$\pm$0.74}
        & 39.87\phantom{.}\scalebox{0.72}{$\pm$0.57}
        & 51.72\phantom{.}\scalebox{0.72}{$\pm$0.88}
        & 66.47\phantom{.}\scalebox{0.72}{$\pm$0.61} \\

        AdaMerging
        & 56.84\phantom{.}\scalebox{0.72}{$\pm$0.86}
        & 36.22\phantom{.}\scalebox{0.72}{$\pm$1.28}
        & \underline{61.38}\phantom{.}\scalebox{0.72}{$\pm$0.71}
        & 42.73\phantom{.}\scalebox{0.72}{$\pm$0.64}
        & \underline{55.23}\phantom{.}\scalebox{0.72}{$\pm$0.60}
        & 69.94\phantom{.}\scalebox{0.72}{$\pm$0.98} \\

        \cmidrule(lr){1-7}

        \textbf{\ourmodel}
        & \textbf{61.96}\phantom{.}\scalebox{0.72}{$\pm$1.12}
        & \textbf{39.48}\phantom{.}\scalebox{0.72}{$\pm$0.88}
        & \textbf{63.84}\phantom{.}\scalebox{0.72}{$\pm$0.60}
        & \textbf{45.08}\phantom{.}\scalebox{0.72}{$\pm$0.67}
        & \textbf{57.99}\phantom{.}\scalebox{0.72}{$\pm$0.84}
        & \textbf{73.30}\phantom{.}\scalebox{0.72}{$\pm$0.80} \\

        \bottomrule

    \end{tabular}
    \end{adjustbox}

    \endgroup

    \vspace{-10pt}
\end{table*}

\clearpage

\section{Prompt Templates}\label{app:prompt}

A task is given the template matching the form of its final answer, not its domain, and the same
one at training and at evaluation. Values use the free-response template, covering mathematics
and the open-ended science of MegaScience and SciBench-Atkins; the multiple-choice GPQA-Diamond
uses the option-letter template; code tasks use the code template. Two benchmarks add a clause to
the free-response instruction: HMMT February 2026 asks for an exact closed form, and
SciBench-Atkins states the expected unit outside \textbackslash{}boxed\{\},
per~\Cref{app:benchmarks-science}.

\definecolor{mathhue}{HTML}{2A9D8F}
\definecolor{codehue}{HTML}{E76F51}
\definecolor{scihue}{HTML}{457B9D}

\newtcolorbox{promptbox}[2]{
  enhanced,
  breakable,
  colback=#1!8,
  colframe=#1!88!black,
  boxrule=0.9pt,
  arc=3pt,
  outer arc=3pt,
  left=8pt, right=8pt, bottom=6pt, top=10pt,
  before skip=5pt, after skip=5pt,
  fonttitle=\bfseries\footnotesize,
  coltitle=white,
  colbacktitle=#1!88!black,
  attach boxed title to top left={xshift=8pt, yshift=-\tcboxedtitleheight/2},
  boxed title style={arc=2pt, outer arc=2pt, boxrule=0pt, left=6pt, right=6pt,
     top=1.5pt, bottom=1.5pt},
  title={#2},
  before upper={\renewcommand{\pfieldhue}{#1!88!black}}
}

\newcommand{\promptbody}{\footnotesize\setstretch{0.9}}
\newcommand{\pgap}{\vspace{2pt}}

\newcommand{\pfieldhue}{black}
\newcommand{\pfield}[1]{\textcolor{\pfieldhue}{\textbf{#1}}}

\lstdefinestyle{promptcodestyle}{
  backgroundcolor=\color{codehue!16},
  basicstyle=\ttfamily\footnotesize,
  frame=none,
  xleftmargin=6pt,
  xrightmargin=6pt,
  framexleftmargin=6pt,
  framexrightmargin=6pt,
  aboveskip=4pt,
  belowskip=0pt,
}

\begin{promptbox}{mathhue}{Value Answers: Math and Open-Ended Science}
\promptbody

Analyze and solve the following [math/science domain] problem step by step.

\pgap
\pfield{Problem:} \textcolor{violet!90}{\textbf{[Insert problem text here]}}

\pgap
\pfield{Hint:} \textcolor{blue!80!black}{\textbf{The tool could be used}} for more precise and
efficient calculations and could help you to verify your result before you reach the final answer.

\pgap
\pfield{Note:} You should first analyze the problem and form a high-level solution strategy, then
utilize the tools to help you solve the problem.

\pgap
\pfield{Answer Format:} Do not put units of the final answer inside \textbackslash{}boxed\{\}. The
content of \textbackslash{}boxed\{\} should be the numerical value of the final answer only, without
any units.

\pgap
Remember once you make sure the current answer is your final answer, do not call the tools again and
directly output the final answer in the following text format, the answer format must be:
\textcolor{red!50!black}{\textbf{\texttt{\textbackslash{}boxed\{'The final answer goes here.'\}}}}.
\end{promptbox}

\begin{promptbox}{scihue}{Option Answers: Multiple-Choice Scientific QA}
\promptbody

Analyze and solve the following [science domain] problem step by step.

\pgap
\pfield{Problem:} \textcolor{violet!90}{\textbf{[Insert problem text here]}}

\pgap
\pfield{Hint:} \textcolor{blue!80!black}{\textbf{The tool could be used}} for more precise and
efficient calculations and could help you to verify your result before you reach the final answer.

\pgap
\pfield{Note:} You should first analyze the problem and form a high-level solution strategy, then
utilize the tools to help you solve the problem.

\pgap
\pfield{Answer Format:} Remember once you make sure the current answer is your final answer, do not
call the tools again and directly output the final answer in the following text format, the answer
format must be:
\textcolor{red!50!black}{\textbf{\texttt{\textbackslash{}boxed\{'The final answer goes here.'\}}}}.
You need to put the final uppercase letter option of this problem into \textbackslash{}boxed\{\}.
\end{promptbox}

\begin{promptbox}{codehue}{Program Answers: Code Generation}
\promptbody

You will be given a question (problem specification) and will generate a correct Python program that
matches the specification and passes all tests.

\pgap
\pfield{Problem:} \textcolor{violet!90}{\textbf{[Insert problem text here]}}

\pgap
\pfield{Public Examples:} Here are some input and output examples of the expected code:

\textcolor{green!50!black}{\textbf{Input:}} [sample inputs]

\textcolor{green!50!black}{\textbf{Output:}} [sample outputs]

\pgap
\pfield{Note:} You should first analyze the problem and form a high-level solution strategy, then
utilize the tools to help you solve the problem.

\pgap
\pfield{Instruction:} Read the inputs from stdin, solve the problem, and write the answer to stdout
(do not directly test on the sample inputs). Enclose your code within the delimiters shown below.
Ensure that when the Python program runs, it correctly reads inputs, executes the algorithm, and
writes output to stdout.

\pgap
\pfield{Submit:} Before submitting your code, you can utilize tools to check its correctness. Once
you make sure the current code is correct, do not call the tools again and submit your code within
the following Python code block:

\begin{lstlisting}[language=Python, style=promptcodestyle]
# YOUR CODE HERE
\end{lstlisting}
\end{promptbox}

\end{document}